\documentclass{article}

\PassOptionsToPackage{numbers, compress}{natbib}
\usepackage[preprint]{neurips_2026}

\usepackage{array}
\usepackage{tcolorbox}
\usepackage[utf8]{inputenc}
\usepackage[T1]{fontenc}
\usepackage{xcolor}
\usepackage{hyperref}
\hypersetup{
   pdffitwindow=true,
   pdfstartview={FitH},
   pdfnewwindow=true,
   colorlinks,
   linktocpage=true,
   linkcolor=blue,
   urlcolor=blue,
   citecolor=blue
 }
\usepackage{url}
\usepackage{booktabs}
\usepackage{amsfonts}
\usepackage{nicefrac}
\usepackage{microtype}
\usepackage{xcolor}
\usepackage{graphicx}
\usepackage{amssymb}
\usepackage{amsmath}
\usepackage{multirow}
\usepackage{makecell}
\usepackage{enumitem}
\usepackage{tabularx}
\usepackage{subcaption}
\usepackage{caption}
\usepackage{amsmath}
\usepackage{amssymb}
\usepackage{mathtools}
\usepackage{amsthm}
\usepackage[capitalize,noabbrev]{cleveref}

\usepackage{algorithm}
\usepackage{algpseudocode}

\theoremstyle{plain}
\newtheorem{theorem}{Theorem}[section]
\newtheorem{proposition}[theorem]{Proposition}
\newtheorem{lemma}[theorem]{Lemma}
\newtheorem{condition}[theorem]{Condition}

\theoremstyle{definition}

\theoremstyle{remark}
\newtheorem{remark}[theorem]{Remark}
\newcommand*\dd{\mathop{}\!\mathrm{d}}
\usepackage{wrapfig}
\usepackage{graphicx}

\title{Learning to Bid in Repeated Second-Price Auctions with  Dynamic Values and  Aggregated Feedback}

\author{%
  Benjamin Heymann\thanks{Equal contribution.} \\
  Criteo AI Lab, Fairplay joint team \\
  \texttt{b.heymann@criteo.com} \\
  \And
  Otmane Sakhi\footnotemark[1] \\
  Criteo AI Lab \\
  \texttt{o.sakhi@criteo.com}
}

\begin{document}
\maketitle

\begin{abstract}
We study the problem of learning to bid when the bidder’s \emph{value is dynamic}, i.e., when the current value depends on past outcomes. Specifically, we consider a bidder participating in repeated second-price auctions whose value depends on the time elapsed since their last successful bid, with auctions arriving in continuous time and only aggregated feedback revealed at the end of the horizon. Such a bidder must \textbf{(1)} balance the immediate benefit of winning the current auction against its impact on future values and \textbf{(2)} learn unknown environmental parameters. We derive regret bounds for a class of learning methods that combine plug-in estimators with a differential-equation characterization of the optimal policy, and show that a specific confidence bound algorithm learns the optimal policy with a near optimal regret of $\widetilde{\mathcal O}(\log N)$ for piecewise linear primitives, and $\widetilde{\mathcal O}(N^{1/3})$ for general, smooth primitives, achieving these regrets without explicit randomization. These theoretical results are supported by numerical experiments.
\end{abstract}

\section{Introduction}
\subsection{Dynamic values}

In a repeated auction,  a bidder's value is \emph{dynamic} when it depends on the bidder's past action.
For example, when the auction is a posted price mechanism and the item being auctioned is a yearly software licence, the buyer  derives no value from purchasing a second licence if they have already bought one. The dynamic value model extends the usual static  one by  accounting for the fact that winning an auction can temporarily reduce the  worth of future opportunities.

Dynamic value matters because it crucially changes the optimal bidding strategy in repeated auctions~\cite{krishna2009auction,10.1145/3447548.3467280}. Perhaps most strikingly, in a second-price auction, bidding the true value  becomes \emph{suboptimal} when values are dynamic.
Hence, ignoring the dynamic dependence can lead to significant economic losses.

 The dynamic value model for repeated auctions is particularly relevant for  digital advertising, where displaying an ad to a user usually diminishes the immediate value of showing another ad to the same user shortly afterward~\cite{10.1145/3447548.3467280,diemert2017attribution,heymann2024pragmatic,bompaire2024fixed}. This phenomenon is particularly evident in real-time bidding (RTB), where automated bidders repeatedly compete for ad slots shown to the same users. Research and industry practice both suggest that users experience \textit{ad fatigue} when exposed to too many similar ads in a short period, leading to diminishing returns for the advertiser. For instance, the probability of a user engaging with an ad (e.g., clicking or making a purchase) tends to decrease if they see the same ad multiple times within a brief window~\cite{heymann2024pragmatic}. Empirically, it is often observed that a bidder's valuation drops sharply after winning an auction and gradually recovers over time. 
 
 This dynamic nature introduces a planning problem: a bidding policy must carefully balance the immediate benefit of winning an auction against the potential long-term cost of devaluing future opportunities. In practice, automated bidders often implement \textit{frequency capping} to limit the number of ads shown to a single user, further validating the relevance of such dynamic considerations in real-world scenarios.
But, as argued in~\cite{heymann2023repeated}, the complexity introduced by dynamic value challenges the common belief that feature engineering and supervised learning are enough to  solve an auto-bidding problem up to optimality. Instead, it often requires a more sophisticated approach, such as solving a planning problem with dedicated tools.

In \cite{heymann2023repeated}, the authors introduce a continuous-time formulation for repeated second-price auctions, derived via a dynamic programming principle. While their work establishes convergence to the Bellman value and synthesizes an optimal control, it assumes full knowledge of the value dynamics and competitive landscape, limiting its practical applicability.

In this work, we extend their framework by analyzing how well a control policy performs when the value dynamics and competition are estimated via plug-in estimators. Specifically, we show that, without requiring prior knowledge of these primitives, several strategies based on plug-in estimation achieve sublinear regret bounds  with high probability. This is the first regret analysis for this setting, demonstrating that even simple estimation can yield theoretically sound performance.
While our analysis remains theoretical, it suggests that combining such approaches with function approximation techniques could pave the way for practical applications in real-world auction systems.

Moreover, given that the underlying problem, characterized by a state variable with resets, departs from standard continuous-time control formulations, we argue that this study offers valuable insights for the broader reinforcement learning and control literature.

\subsection{Contributions}
Our approach establishes novel connections between continuous-time optimal control, model-based reinforcement learning, and auction theory. Unlike much of the model-based RL literature, our work does not assume the reward function is known. Building on insights from~\cite{heymann2023repeated}, we study a model with an unusual structure involving resets.

We analyze four different algorithms based on plug-in estimators. The first algorithm directly applies plug-in estimators learned from past data to the solver. We show that this algorithm can achieve arbitrarily close approximation to the optimal value given sufficient time. This result follows as a corollary of the extension invariance theorem~\ref{th:extension-theorem}, which we believe constitutes one of our main contributions.
The second and third algorithms resemble what practitioners would likely prefer in high-stakes production environments: a small number of iterated rollouts. We prove that both variants achieve sub-linear regret.

Finally, we propose an online algorithm based on confidence intervals that achieves logarithmic regret.
Interestingly, the algorithm does not rely on a randomization component despite the need to learn the cumulative distribution of the competition.
We further establish a logarithmic lower bound on the minimax regret, demonstrating that our online learning algorithm attains the optimal rate.
Finally, we compare three algorithms experimentally.

To the best of our knowledge, this is the first work to explicitly account for the dynamic effect of bids on future values in a learning setting.


\section{Related Work}

The interplay between optimal bidding in repeated auctions and learning has been extensively studied, with researchers proposing a range of bandit-based approaches to address the challenge~\cite{achddou2021efficient,weed2016online}.
It is notable that in~\cite{weed2016online}, which corresponds to a static version of our setting, the authors derive a tight logarithmic lower bound under regularity assumptions similar to ours.
Over the past few years, a variant of the problem with budget constraints has attracted a lot of attention~\cite{balseiro2019learning,balseiro2020dual,castiglioni2022unifying}.

The problem of sequence effects, where the outcome of one auction influences future bids, has also garnered attention. Early efforts formalized this challenge using a Markov Decision Process (MDP) framework, though the resulting policy optimization problem proved computationally intractable, necessitating heuristic solutions~\cite{10.1145/3447548.3467280}. Subsequent work has focused on practical improvements: \cite{heymann2024pragmatic} demonstrated a first-order method in production, while~\cite{betlei2024maximizing} introduced a reparameterization technique to map timeline effects back to a contextual bandit setting. This latter approach was  refined in~\cite{heymann2025non}, offering a more flexible parametrization.

Our approach is philosophically close to model-based reinforcement learning in continuous time~\cite{guo2023reinforcement,yildiz2021continuous}, that mixes ideas from continuous time RL~\cite{wang2020reinforcement}, with model-based approaches. It is notable that the authors of~\cite{guo2023reinforcement} derived an algorithm with logarithmic regret
in the number of episodes. Interestingly, our problem does not fit their framework  in particular because of the  random, non-linear rewards and  of the switching events that depend on the control.

Close to our application domain, \cite{badanidiyuru2022incrementality} also justify the use of "mixed" and "delayed" feedback as a condition to account for causal effects. Following a landmark paper in the field of incrementality for advertising~\cite{lewis2022incrementality}, they also use the time gap from the  last ad impression as a state. Our results differ in that in our method, we do not need exploration of the states and provide a method with near logarithmic regret, while they use an UCB algorithm to get a regret in \( O(\sqrt{N}) \). Like us, they propose an episodic formulation, where each episode corresponds to a repeated interaction with an internet user. A notable difference is that they use a discrete time model where the number of auctions is given in advance, which abstracts away a notable practical challenge. 
A different ad optimization problem was investigated in~\cite{darmasubramanian2025ads}, where the authors consider the problem of reallocating ads over a window of known length. They also motivate their investigation with a behavioral model of the user (different from ours); however, their setting does not consider an auction.

\section{Setting}
\label{sect:setting}
\subsection{Dynamic values}

\label{subsect:dynamic_value}

Our setting adapts the  dynamic value  framework with concave reward function, stationary competition, and second-price auction  from~\cite{heymann2023repeated}. The initial input data corresponds to an offline reinforcement learning dataset, where each episode consists of a sequence of second-price auctions, each generated by a policy \( \pi_0: t \in \mathbb{R}^+ \to [0,1] \). The length of each episode is drawn from a Poisson distribution with rate parameter \( \gamma \), independent of both the arrival process of auction requests and the outcomes of individual auctions. Auction arrivals within an episode follow an independent Poisson process with intensity \( \mu \).

The competition faced by the decision maker is stationary and \textit{i.i.d.}, with the highest competing bid (price-to-beat) for each auction drawn independently from the continuous cumulative distribution function (CDF) \( q \). For a bid \( b \), \( q(b) \) represents the probability that \( b \) exceeds the price-to-beat, thus securing a win. The classic \textit{i.i.d.} assumption is justified in contexts where the bidder's market share is negligible compared to the overall market size, ensuring that individual auction outcomes do not influence the distribution of future competition.

At the end of each episode, the decision maker observes the total value
generated by all auctions in that episode. This value is given by:
\begin{align}
\label{eq:obj}
    \sum_{i=1}^\ell
    \left(
        k(t_i - T_{i-1}) + \varepsilon_i - b_i^-
    \right)
    \cdot
    \mathbb{I}[b_i > b_i^-],
\end{align}
where \( \ell \) is the random number of auctions in the episode, \( b_i \) is the bid placed by \( \pi_0 \) for the \(i\)th auction, \( b_i^- \) is the price-to-beat for the \(i\)th auction, \( t_i \) is the time of the \(i\)th auction, \( T_{i-1} \) is the time of the last auction won in the episode, so
    \(t_i - T_{i-1}\) is the elapsed time since the last win, \( k(\cdot) \) is a concave and increasing function mapping the
    elapsed time since the last win to the expected reward, \( \varepsilon_i \) is a zero-mean sub-Gaussian noise term,
    conditionally independent across auctions, capturing fluctuations
    around the mean reward \(k(t_i-T_{i-1})\), and \( \mathbb{I}[b_i > b_i^-] \) is the indicator of winning the \(i\)th auction.

We suppose that each episode starts with an initial $T_0 = 0$.
Given a bidding policy $\pi$, we denote by $V_\pi(0)$ the resulting expected payoff.

The term \( k(t_i - T_{i-1}) - b_i^- \) thus represents the average net profit from winning the auction: the reward \( k(t_i - T_{i-1}) \) minus the payment \( b_i^- \).
In contrast to the setting in~\cite{heymann2023repeated}, where the functions \( q \) and \( k \) are assumed known, the decision maker must here infer both \( q \) and \( k \) from observed data.
However, even when the data are knwon, the resulting optimal control problem, with its reset condition, possesses an unusual structure that motivates specific tools for its resolution. 

\textbf{We focus on the learning of \( k \) and \( q \)}, and suppose \( \mu \) and \( \gamma \) known to the decision maker. This is mostly because their estimations do not depend on the bidding policy of the decision maker, but only rely on observational data, for which estimators are available. 

\subsection{Preliminaries, Theoretical solution and algorithm}
\label{subsec:theoretical-solution-and-algorithm}
We first consider the problem of maximizing criterion~\eqref{eq:obj} given the parameters \( k \) and \( q \).
Because the auction follows a second-price rule, the price paid on average when bidding \( b \) can be inferred from \( q \) using the relation $p(b) =  q(b)b - \int_0^bq(s)\dd s.$ 

The Bellman value \( V^\star(\tau) \) represents the maximal expected total value given that the time elapsed since the last won auction is \( \tau \).
It is a standard method in optimal control to first derive the Bellman value and then use it to derive an optimal bid.

The next result, adapted from~\cite{heymann2023repeated}, provides a characterization of the Bellman value.
Instead of solving a Cauchy problem where the initial condition is given, one must derive the unique initial condition that makes the solution bounded, a situation reminiscent of the shooting method in optimal control. They rely on classical results on dependence on initial conditions and parameters of ordinary differential equation (see for Chapter 5 in ~\cite{hartman2002ordinary}).

\begin{theorem}[From~\cite{heymann2023repeated}]
\label{lemma:cauchy}
Let \( f(v) = q(v)v - p(v) \) and 
\(
\Phi(t, v, \lambda) = \gamma v - \mu f(k(t) + \lambda - v).
\)
 The Bellman value function \( V^\star(.) \) is the unique bounded element among
 the solutions of the following family of Cauchy problems parameterized by \( y_0 \in \mathbb{R}_+ \):

\[
\tag{$\mathcal{F}_{y_0}$}
\label{eq:cauchy}
\left\{
\begin{array}{ll}
\dot{Y}(t) = \Phi(t, Y(t), y_0) \\
Y(0) = y_0.
\end{array}
\right.
\]
\end{theorem}
In what follows, we denote by \( Z^{y_0} \) the solution of \( \mathcal{F}_{y_0} \).
Next adapting another result from~\cite{heymann2023repeated}, we note that one can derive an optimal policy if given the Bellman value at \( 0 \).

\begin{lemma}[adapted from~\cite{heymann2023repeated}]
\label{lemma:b_dynamics}
If \( k \) is piecewise \( C^1 \)  then, wherever $k$ admits a derivative,  an optimal policy \( \pi^\star \) satisfies the differential equation 
\[
      \dot{\pi}(t) = \dot{k}(t)-\gamma(V^\star(0)+k(t)) +\mu f(\pi(t))+\gamma \pi(t).
\]   
\end{lemma}
\textbf{Proof:} Appendix~\ref{proof:lemma:b_dynamics}.
\newline

The solver $\mathcal{S}: k,q \rightarrow \pi$ introduced by~\cite{heymann2023repeated} uses a bisection method to estimate the value of the Bellman function at 0. It iteratively refines a guess by observing the asymptotic behavior of the corresponding ordinary differential equation (ODE) solution: if the solution diverges to \( -\infty \), the guess is increased; if it diverges to \( +\infty \), the guess is decreased. Combined with Theorem~\ref{theorem:synthesis}, this solver allows us to recover the optimal policy $\pi^\star$.


\begin{theorem}[From~\cite{heymann2023repeated}]
\label{theorem:synthesis}
Let $y$  be the output of the solver as \( N \) goes to $+\infty$, then  $Z^{y}(.)=V^\star(.)$. Moreover, $\pi^{\star}(\tau)=\max \left(0 ; k(\tau)+V^{\star}(0)-V^{\star}(\tau)\right) .$
\end{theorem}

\section{Learning tools}


While~\cite{heymann2023repeated} assumed oracle access to \( q \) and \( k \), and focuses on deriving an optimal bidding policy using a method based on differential equations, here, we suppose that we have access to an oracle that, given \( q \) and \( k \), provides us with an optimal solution, but that \( q \) and \( k \) need to be estimated from the data. In what follows, we denote by
$\mathcal{S}(k,q)$ an optimal policy for input data \( k \) and \( q \).

Let $k^\star$ and $q^\star$ be the true primitives, 
our goal is to sequentially build the bidding policies $\pi_1\ldots \pi_N$ to minimize the regret defined over a horizon of $N$ episodes as $\operatorname{Regret}(N) = \sum_{l=1}^N V_{\pi^\star}(0) - V_{\pi_i}(0).$

\subsection{Further assumptions and notations}

An insight from~\cite{heymann2023repeated} is that for the bid to be well-behaved, \( k \) must be \emph{concave}. This requirement prevents us from using alternative hypothesis sets, such as the set of step functions for \( k \).
In what follows, we suppose \( k \) and \( q \) are piecewise linear functions with known breakpoints $(t_i)_{i\in [0,I_k]}$  and $(b_i)_{i\in [0,I_q]}$ to study the core of the problem and extend our results to general primitives in Section~\ref{sec:gen_primitive}.

 Let $\mathcal{K}$ denote the set of \textbf{piecewise linear, concave, and increasing} functions, whose breakpoints are in $(t_i)_{i\in [0,I_k]}$.
Similarly let $\mathcal{Q}$ denote the set of \textbf{piecewise linear, increasing, and bounded} functions, with breakpoints also restricted to $(b_i)_{i\in [0,I_q]}$.
We assume first that $k \in \mathcal{K}$ and $q \in \mathcal{Q}$ and then extend our result to general primitives. 
We assume there exists a constant $\alpha > 0$ such that for any bid $b \in [0, k(t_{I_{k}-1})]$, the cumulative distribution function \( q \) satisfies $q(b) > \alpha b.$

Given a policy $\pi:t\in[0,+\infty[\to b\in \mathbb{R}_+$, let $\beta(\pi) = \lim_{t\to+\infty}\pi(t)$ denote its asymptotic bid value.
For a given policy $\pi$, we denote by $I(\pi)$ the smallest index $i$ such that $\beta(\pi)\leq b_i$.
For $k\in \mathcal{K}$, $q\in\mathcal{Q}$, $\pi\in\mathcal{P}$ and $t\in \mathbb{R}_+$, we denote by $\mathcal{U}(k,q,\pi,t)$ the expected payoff of policy $\pi$ when evaluated in the environment $(k,q)$ and starting at time $t$.
\subsection{Extension invariance theorem}
We prove a core result below and believe it elegantly captures the interplay between the second-price auction mechanism, the dynamic optimization problem, and the learning problem.

\begin{theorem}[Extension invariance]
\label{th:extension-theorem}
Let $k\in \mathcal{K}$ and $(\tilde{q},\hat{q})\in \mathcal{Q}^2$ such that $\tilde{q}(b) = \hat{q}(b) $ for all $ b\in [0,\beta(S(k,\hat{q}))]$, we have the identity
\begin{align*}
             \mathcal{U}(k,\hat{q},\mathcal{S}(k,\hat{q}),0) 
      = 
    \mathcal{U}(k,\tilde{q},\mathcal{S}(k,\tilde{q}),0).
\end{align*}
\end{theorem}
\textbf{Proof:} See Appendix~\ref{proof:extension-theorem}.
\newline
Theorem~\ref{th:extension-theorem} has significant implications for learning, as it shows that accurate estimation of \( q^\star \) is only required on a subset of the domain, rather than on the entire space. Specifically, the theorem guarantees that knowledge of \( q^\star \) up to the maximal optimal bid \( \beta(\mathcal{S}(k^\star, q^\star)) \) is sufficient to identify the optimal policy.

\subsection{Plug-in Estimator for \( k \)}
\label{subsubsection:k}

\paragraph{Piecewise Linear Representation.} We parameterize \( k \) as a concave and increasing piecewise linear function with coefficients $(a_i)_{i \in [I_k]}$: $\forall \tau \ge 0\,, k(\tau) = \sum_{i = 1}^{I_k}a_i \cdot (t_i \wedge \tau - t_{i-1} \wedge \tau )\,$
with $a_{I_k} = 0$, $t_{I_{k}-1} = \bar{T}$ and $t_{I_k} = +\infty$.

\paragraph{Linear Regression Formulation.} Each episode provides an observation relating the cumulative values of \( k \) at different ages $\tau$ (time elapsed since last win) to the total episode value. Aggregating \( N \) episodes, we obtain a linear system where each episode $n$ contributes an equation. The \emph{age} $\tau$ at each auction is the time elapsed since the last won auction in that episode.

\paragraph{Data Representation.} Suppose that after $N$ episodes, we have access to a dataset \( D_N^k = \{\tau^1_n, \cdots, \tau^{\ell}_n, v_n \}_{n \in [N_e]} \), where each \(n\) represents an independent episode (a sequence of auctions before the episode ends). Here $N_e \le N$ is the number of episodes with at least one won auction. This dataset can be transformed into a matrix form \( \{X_n = [x_{1, n}, \cdots, x_{I, n}], v_n \}_{n \in [N_e]} \).  This allows us to express the relationship between the observed values and the coefficients as:
\[
\forall n, \quad v_\ell = X_n^t a + \epsilon_\ell \quad \iff \quad V = X^t a + \epsilon,
\]
where \(\epsilon_n\) is a centered, conditional sub-Gaussian noise term.

\paragraph{Estimation and Projection}

Let \(\bar{a} = (X X^t)^{-1} X^t V\) be the ordinary least squares (OLS) estimator of \( a \). Define the convex set $\mathcal{C} = \left\{ a \in \mathbb{R}^I \mid 0 = a_I \leq a_{I-1} \leq \cdots \leq a_1 \right\}$, which encodes the monotonicity and concavity constraint. The true parameter \( a \) is assumed to lie in \( C \). We then define the projected estimator $\hat{a} = \textsc{proj}_{\mathcal{C}}(\bar{a}),
$ where \(\textsc{proj}_C\) denotes the projection onto the convex set \( C \).
Given a dataset $D_N^k$, we denote by $\hat{k} = K(D_N^k)$ the resulting estimator of \( k \).

\begin{proposition}[Concentration of $k$ - Informal] \label{prop:concentration_k_main}
Let $\delta \in (0,1 )$ and $\epsilon > 0$. The bidding policies $\pi_n$ are allowed to depend on the past. If the bidding policies used to collect the data $D_N^k$ constructed from $N$ episodes verify $\pi_n(t_1) \ge \epsilon$ then we have:
\begin{align*}
    \max_{\tau \ge 0} |k^\star(\tau) - \hat{k}(\tau)|\le  \mathcal{O}\left(1/\sqrt{N}\right)\,.
\end{align*}
\end{proposition}
\textbf{Proof:} Appendix~\ref{prop:concentration_k}

\begin{remark}
Given the structure of the problem, learning \( k \) requires only a simple condition: bidding above a threshold \( \epsilon \) at $t_1$ is sufficient to guarantee convergence of the estimator.
\end{remark}

\subsection{Plug-in Estimator for \( q \)}
\label{subsubsection:q}

As shown in~\cite{heymann2023repeated}, access to \( q \) is necessary to achieve the optimal policy. This introduces additional complexity compared to the static setting, where \( q \) does not need to be learned. Indeed, as demonstrated in~\cite{weed2016online}, the static problem is truthful, and the optimal policy simply bids the value \( k \).

\paragraph{Representation and Dataset}
We parameterize the win probability \( q \) as a piecewise linear, increasing function with slope coefficients $(c_i)_{i \in [I_q]}$: $
    \forall v \in [0,1]\,,  q(v) = \sum_{i = 1}^{I_q}c_i \cdot (b_i \wedge v - b_{i-1} \wedge v)\,,$
where $c_i > 0$ for all $i$ to ensure monotonicity, and $b_i$ its breakpoints. 

In second-price auctions, we observe wins but also the price paid when we win. Our dataset is therefore $\mathcal{D}^q = \left\{b_n, w_n, p_n \right\}_{n \in [N_b]},$
where $b_n$ is the bid, $w_n \in \{0,1\}$ indicates a win, $p_n$ is the price paid (with convention $p_n = 0$ when $w_n = 0$), and $N_b$ is the number of auctions observed.

\paragraph{Maximum Likelihood Estimation}
Since \( q \) is also the cumulative distribution of the competition with density $f_c(v) = \sum_{i=1}^{I_q} c_i \cdot \mathbb{I}[v \in (b_{i-1}, b_i]]$, we estimate the slope vector $c$ via maximum likelihood as the maximizer $\hat{c}$ of   
\begin{align*}
 \sum_{n \in [N_b]} \mathbb{I}[w_n = 1] \log f_c(p_n) + \mathbb{I}[w_n = 0] \log \left(1 - q_c(b_n) \right).
\end{align*}
This is an easy \emph{concave optimization problem}. MLE exploits the additional price information $p_n$ from second-price auctions, achieving efficient estimation compared to methods based solely on win/loss outcomes. Given a dataset $D_N^q$, we denote by $\hat{q} =Q(D_N^q)$ the resulting estimator of \( q \).

As stated by the extension theorem~\ref{th:extension-theorem}, knowledge of \( q \) up to the maximal optimal bid $\beta(\mathcal{S}(k^\star,q^\star))$ is sufficient to learn the optimal policy. Therefore, it is enough to focus on this region of the bid space and ensure convergence of the estimator there. The following result provides a sufficient condition for the concentration of our estimator of \( q \).

\begin{proposition}[Concentration of $q$ - Informal]\label{prop:concentration_q_main}
Let $\delta \in (0, 1)$ and $T > 0$. The bidding policies $\pi_n$ are allowed to depend on past observations. If the bidding policies used to collect the data $D_N^q$ constructed from $N$ episodes verify that $\forall t \ge T, \pi_n(t) \ge \beta(\mathcal{S}(k^\star,q^\star))$, then we have:
\begin{align*}
    \forall v \in [0, \beta(\mathcal{S}(k^\star,q^\star))],\, |q^\star(v) - \hat{q}(v)| \le \mathcal{O}\left(1/\sqrt{N} \right)\,.
\end{align*}

\end{proposition}
\textbf{Proof:} See Appendix~\ref{prop:concentration_q}
\begin{remark}
Constructing bidding policies that cover the maximal optimal bid over a non-negligible time interval is sufficient to ensure efficient learning of \( q \) on the region of interest.
\end{remark}

\section{Learning algorithms}
\label{sect:analysis}

\begin{wrapfigure}[8]{o}{0cm}
\vspace{-2.2\baselineskip}
\begin{minipage}[t]{0.415\textwidth}
\vspace{-2.2\baselineskip}
\begin{algorithm}[H]
\caption{$\textsc{Offline-learning}(\mathcal{D})$}
\label{algo2}
\begin{algorithmic}[1]
\Require Dataset $\mathcal{D}=(D^k,D^q)$
\State Set $\hat{k} = K(D^k)$
\State Set $\hat{q} = Q(D^q)$
\State Set $\hat{\pi}= \mathcal{S}(\hat{k},\hat{q})$
\State Rollout $\hat{\pi}$
\end{algorithmic}
\end{algorithm}
\end{minipage}
\vspace{-1em}
\end{wrapfigure}

The model's two primitives, \( q \) and \( k \), require \emph{distinct treatments} due to differences in their identifiability.
On one hand, the valuation \( k \) can be estimated for bids that win with non-zero probability. Observed winning bids provide direct information about the underlying valuation distribution, allowing for consistent estimation in regions where bids are competitive.
In contrast, the competition primitive \( q \) is only identifiable up to the highest observed bid threshold. If the bidding policy never submits bids above a certain value, the competitive environment beyond that threshold remains unobserved. This fundamental asymmetry between the identifiability of \( k \) and \( q \), in connection with Theorem~\ref{th:extension-theorem} shapes the design and analysis of our learning algorithms. 


We now present four learning algorithms with different theoretical guarantees
and practical merit. They all share the structure described in
Algorithm~\ref{algo2}: data is generated, used to learn possibly biased
estimators, which are then plugged into the solver to produce a policy.
This process is then repeated.



\subsection{Asymptotic Convergence without Regret Guarantees}

Our first result establishes that a simple iterative plug-in approach converges to the optimal policy without requiring explicit exploration, though it does not provide finite-time regret bounds.

\begin{theorem}[Asymptotic Convergence]
\label{th:asymp}
Consider initial datasets $\mathcal{D}^k_{0}$, $\mathcal{D}^q_{0}$ generated by a policy $\pi_0$, and the iterative sequence:
\begin{align*}
    \pi_n &= \mathcal{S}(K(\mathcal{D}^k_n), Q(\mathcal{D}^q_n))\\
    \mathcal{D}^k_{n+1} &= \mathcal{D}^k_{n} \cup \mathcal{D}^k_{\pi_n}\\
    \mathcal{D}^q_{n+1} &= \mathcal{D}^q_{n} \cup \mathcal{D}^q_{\pi_n},
\end{align*}
where $\mathcal{D}^{.}_{\pi}$ is a dataset generated by policy $\pi$ over a strictly positive number of episodes.
If
\begin{align*}
   \limsup_{n \to \infty} \beta(\pi_n) > 0,
\end{align*}
then for any $\varepsilon > 0$, $\pi_n$ is $\varepsilon$-optimal for $n$ sufficiently large.
\end{theorem}
 Theorem~\ref{th:asymp} is a consequence of Theorem~\ref{th:extension-theorem}, through Lemma~\ref{lemma:epsilon-opt}.
 
\textbf{Proof:}
See Appendix~\ref{proof:th:asymp}
\begin{remark}
The condition $\limsup_{n \to \infty} \beta(\pi_n) > 0$ ensures that the sequence of policies explores a region where the competition distribution \( q \) is identifiable. This theorem shows that no explicit randomization or exploration is required: the policy naturally converges to optimality through greedy plug-in estimation alone. 
Another way to understand this result is that  self-confirming policies are optimal.
However, this result does not provide finite-time regret guarantees, as convergence may be slow in the initial phases.
\end{remark}

\paragraph{Guidelines for Algorithms with Sub-Linear Regret.} We give the following core result that links $\epsilon$-optimality of the estimators $q$ and $k$ to the optimality of $\pi = \mathcal{S}(k,q)$ and characterizes the regret.
\begin{proposition}[$\epsilon$-optimality]\label{prop:epsilon_opt} Let $\epsilon >0$, $q$ and $k$ such that $\lVert k^\star - k \rVert_\infty \le \epsilon$ and $\max_{[0, \beta(\pi^\star)]} |q^\star(v) - q(v)| \le \epsilon$. We have for small $\epsilon$:
    \begin{align*}
    &\lVert \pi^\star- \pi\rVert_{\infty} \le \mathcal{O}(\epsilon)\,,\quad
    &\lVert V^\star-V^{\pi}\rVert_{\infty}\leq \mathcal{O}(\epsilon^2).
\end{align*}
\end{proposition}

\textbf{Proof:} This is the result of combining Lemmas~\ref{lemma:reg_quad} and~\ref{lemma:eps_pol}.

While Theorem~\ref{prop:epsilon_opt} does not directly yield a no-regret algorithm, it provides the final ingredient needed to establish regret guarantees. In particular, \( \epsilon \)-optimality can be achieved with a sufficiently accurate estimate of \( k \) and an estimate of \( q \) that is accurate up to \( \beta(\mathcal{S}(k^\star, q^\star)) \). To ensure this, and as stated in Propositions~\ref{prop:concentration_k_main} and~\ref{prop:concentration_q_main}, the bidding policies must bid sufficiently high at time \( t_1 \) and cover the maximal optimal bid \( \beta(\mathcal{S}(k^\star, q^\star)) \). The difficulty is that this quantity is unknown, which forces sub-linear regret strategies to exploit the structure of the problem in order to guarantee adequate coverage of this region.

\subsection{Two-phase and Three-phase Algorithms}



Our second and third algorithms implement the idea above by simply bidding higher than \( \beta(\mathcal{S}(k^\star, q^\star)) \) for a sufficient number of sessions, using an approach reminiscent of the explore-then-commit strategy from bandit problems. We show that these simple algorithms achieve sub-linear regret.


\begin{theorem}[Two-phase algorithm]\label{th:twostep-main}
Consider the \textit{ Learn-Then-Rollout} algorithm:
\begin{enumerate}
    \item \textbf{Exploration Phase ($N_1$ episodes):} Use a exploratory policy $\pi_{\text{x}}$ that bids $k(\infty)$ (the maximal value) to collect data for estimating \( k \) and \( q \). Compute $\hat{k} = K(\mathcal{D}^k_{N_1})$ and $\hat{q}=Q(\mathcal{D}^q_{N_1})$.
    \item \textbf{Exploitation Phase ($N_2 = N - N_1$ episodes):} Use policy $\pi = \mathcal{S}(\hat{k}, \hat{q})$.
\end{enumerate}    

Let $\delta \in (0, 1)$. There exists $C_0$ such that once $N \geq C_0$, setting $N_1 = \mathcal{O}(\sqrt{N})$ yields with probability at least $1-\delta$ that $\operatorname{Regret}(N) \leq \mathcal{O}(\sqrt{N}).$
\end{theorem}




\textbf{Proof:}
See Appendix~\ref{th:appendix:2phase}

\begin{theorem}[Three-phase Algorithm]\label{th:threestep-main}
Consider the \textit{Learn-$k$-Then-$q$-Then-Rollout} algorithm:
\begin{enumerate}
    \item \textbf{Exploration Phase for $k$ ($N_1$ episodes):} Use a fixed exploratory policy $\pi_{\texttt{x}}$ that bids $b_0 > 0$ to collect data for estimating \( k \). Compute $\hat{k} = K(\mathcal{D}^k_{N_1})$.
    \item \textbf{Exploration Phase for $q$ ($N_2$ episodes):} Use the learned $\pi_k =  \hat{k}$ to collect data for estimating \( q \). Compute $\hat{q} = Q(\mathcal{D}^q_{N_2})$.
    \item \textbf{Exploitation Phase ($N_3 = N - N_1 - N_2$ episodes):} Use policy $\pi = \mathcal{S}(\hat{k}, \hat{q})$.
\end{enumerate}    
Let $\delta \in (0, 1)$. There exists $C_0$ such that once $N \geq C_0$, setting $N_1 = N_2 =  \mathcal{O}(\sqrt{N})$ yields with probability at least $1-\delta$ that $\operatorname{Regret}(N) \leq \mathcal{O}(\sqrt{N}).$
\end{theorem}



\textbf{Proof:}
See Appendix~\ref{proof:three-phase}


In the explore phase, these algorithms are ensured to bid higher than $\beta(\pi^\star)$ as $k$ covers the optimal bid (Theorem~\ref{theorem:synthesis}). Both algorithms achieve the fast $\mathcal{O}(\sqrt{N})$ convergence rate. From a practical perspective, this means that the choice of how to implement an explore-and-commit strategy—such as when transitioning from a static to a dynamic bidding model—does not hinge on statistical efficiency but on practical constraints. In particular, bidding the value, whether treated as static or time-dependent, is sufficient to explore the relevant part of the competitive landscape and to learn a near-optimal policy for the dynamic setting.

\subsection{UCB-Style Algorithm with Biased Estimates}

Our fourth algorithm takes inspiration from the Upper Confidence Bound (UCB) principle, using optimistic and pessimistic estimates to balance exploration and exploitation.

We rely on the following key intermediate result.  
\begin{theorem}[Monotony of $\beta$]
\label{th:monotony-theorem}
Let $k_1,k_2\in \mathcal{K}$, such that $k_1\leq k_2$ and 
$(q_1,q_2)\in \mathcal{Q}$ such that $q_1\geq q_2$ then 
$$\beta(\mathcal{S}(k_1,q_1))\leq \beta(\mathcal{S}(k_2,q_2))\,.$$
\end{theorem}
\begin{proof}
    We combine Lemma~\ref{lemma-monotony-beta-k} and Lemma~\ref{lemma-monotony-beta-q}.
\end{proof}

\begin{theorem}[Confidence Bounds Algorithm]
\label{th:ucb-main}
Choose $\lambda_k, \lambda_q >0$. For each episode $n \in [1, N]$ do:
    \begin{itemize}
        \item $\hat{k}^{\texttt{UCB}}_n = \min(K(\mathcal{D}^k_n) + \lambda_k \sqrt{\log\log n/n}\,,\, k(\infty))$
        \item $ \hat{q}^{\texttt{LCB}}_n = \max(Q(\mathcal{D}^q_n) - \lambda_q \sqrt{\log \log n/n} \,,\, 0)$
        \item $\pi_n = \mathcal{S}(\hat{k}^{\texttt{UCB}}_n, \hat{q}^{\texttt{LCB}}_n)$
        \item $\mathcal{D}^k_{n+1} = \mathcal{D}^k_{n} \cup \mathcal{D}^k_{\pi_n}$ and $\mathcal{D}^q_{n+1} = \mathcal{D}^q_{n} \cup \mathcal{D}^q_{\pi_n}$ with $\mathcal{D}^k_{\pi_n}$ and $\mathcal{D}^q_{\pi_n}$ collected at episode $n$.
    \end{itemize}   
The algorithm constructs an envelope around $q$ and $k$ and derives an optimal policy for the plausible worst case environment.

Let $\delta \in (0, 1)$. There exist $\lambda_q$ and $\lambda_k$ (depending on problem constants) that makes the algorithm yield with probability at least $1-\delta$ that $\text{Regret}(N) \leq \tilde{\mathcal{O}}(\log(N)).$
\end{theorem}




\textbf{Proof:}
See Appendix~\ref{proof-ucb}

In the following section, we establish that this algorithm achieves near-minimax optimality. This algorithm shares structural similarities with the first algorithm, the key distinction lies in the use of biased plug-in estimators rather than unbiased ones.

 In addition, optimizing over $q$
 does not require randomization, which is typical when learning to bid in first-price auctions. In contrast, the static second-price case does not provide an interesting point of comparison, since the truthfulness property of static second-price auctions implies that knowledge of the bid distribution is unnecessary~\cite{weed2016online,krishna2009auction}.

\subsection{Lower Bound}
We consider two degenerate cases where $k$
 is constant.
In doing so, the problem reduces to the static case, with the important distinction that feedback arrives in batches and is aggregated.
Since allowing the algorithm to de-aggregate the data freely can only improve its performance, any lower bound derived under this relaxation remains valid for the original problem.
This reduction maps our setting to one of the scenarios analyzed in~\cite{weed2016online}, for which a lower bound has already been established. More precisely, we can apply their Theorem 4 with $\alpha=1$ to get the next result.
\begin{theorem}
\label{th:lowerbound}
There exists $C>0$ such that
    for any algorithm, and any $N$, $\operatorname{Regret}(N)\geq C\log(N)$.
\end{theorem}
Theorem~\ref{th:lowerbound} implies in particular, that our last algorithm is near minimax, and, surprisingly, that the learning problem when primitives are piecewise linear is not much harder than the static case.

\subsection{General Smooth Primitives}\label{sec:gen_primitive}

We now extend the previous guarantees, proved for piecewise linear
primitives, to smooth primitives by approximation and obtain a regret of $\widetilde{\mathcal O}(N^{1/3})$ for our confidence bounds algorithm.

\begin{lemma}[Piecewise linear approximation]
\label{lem:piecewise_linear_approximation}
Let \(f:[a,b]\to\mathbb R\) be continuously differentiable and assume that
\(f'\) is \(L\)-Lipschitz. Let \(f_{PL,m}\) be the piecewise linear
interpolant of \(f\) on a uniform grid of \(m\) intervals on \([a,b]\).
Then $\|f-f_{PL,m}\|_\infty
    \le
    \mathcal{O}(1/m^2).$
\end{lemma}

\begin{theorem}[Regret for general smooth primitives]
Suppose that \(k\) and \(q\) are continuously differentiable and have
Lipschitz derivatives on their respective domains. Let \(N\) be the horizon,
and use the family of piecewise linear functions for \(k\) and \(q\) on uniform
grids with $m \asymp N^{1/6}$
intervals. 

Then the Confidence Bounds Algorithm satisfies, with high
probability, $\operatorname{Regret}(N)
    \le
    \widetilde{\mathcal O}\!\left(N^{1/3}\right).$
\end{theorem}

\textbf{Proof.} See Appendix~\ref{thrm:regret_general}

This result shows that the Confidence Bounds Algorithm extends beyond the
piecewise linear setting. In particular, the method applies to general smooth
primitives and does not require the primitives to be piecewise linear or the
locations of their breakpoints to be known in advance.

\section{Experiments}
\label{sect:experiments}


\begin{wrapfigure}{r}{0.49\textwidth}
\centering
\vspace{-\baselineskip}
\includegraphics[
    trim=.3cm .25cm .25cm .25cm,
    clip,
    width=0.45\textwidth
]{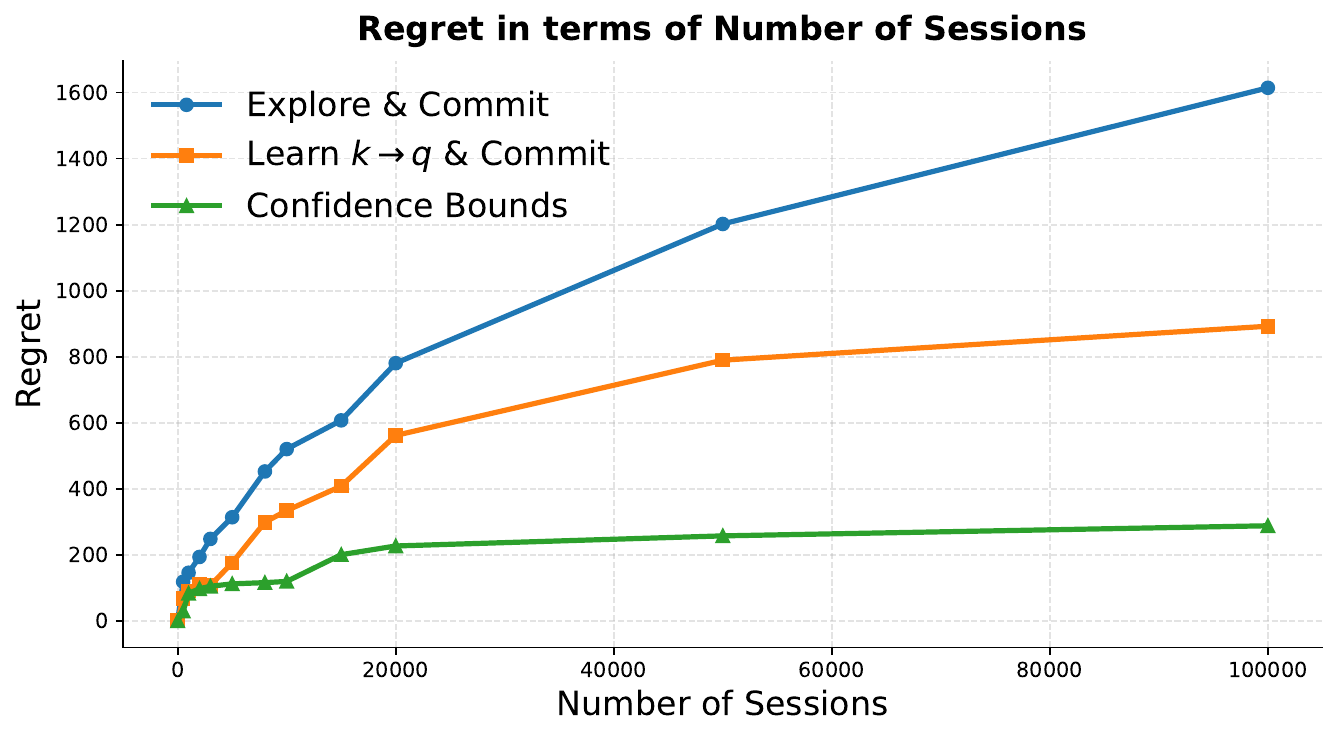}
\caption{Regret as a function of \(N\).}
\label{fig:regret}
\vspace{-1em}
\end{wrapfigure}
In this section, we design a simple experimental setting to validate our algorithms.
We consider a single environment with
\( k(t) = 1 - \exp(-\theta t) \) and \( q(v) = v^{\alpha} \), where
\( \alpha = 2 \) and \( \theta = 0.1 \).
We use a family of piecewise linear functions with 10 uniform breakpoints as an approximation.
We set \( \mu = 0.5 \) and \( \gamma = 0.1 \), and assume that these parameters are known.
We evaluate the no-regret algorithms on this environment. To compute the optimal policy, we
numerically solve the associated ODE using the \texttt{SciPy} ODE solver with the default
Runge--Kutta method. The results are reported in Figure~\ref{fig:regret}, where each
strategy is re-run for every fixed horizon \(N\); they are overall consistent with our
theoretical results.

\section{Limitations and Future Work}
\label{sect:discussion}

In this work, we studied learning to bid with dynamic values, focusing on
second-price auctions where the bidder's value depends on the time since their
last win. Natural extensions include other auction formats, whose strategic
features may interact differently with our resolution method. Another direction
is to enrich the value model beyond recency,  which may be too restrictive for some applications, allowing values to depend on the
full history of won auctions, patterns of wins and losses, past prices paid, or
random item-specific components capturing heterogeneity across objects. We believe these
directions could make the model richer and relevant beyond computational advertising, including
financial applications.

\bibliographystyle{unsrt}
\bibliography{paper.bib}

@article{heymann2024pragmatic,
  title={A pragmatic policy learning approach to account for users' fatigue in repeated auctions},
  author={Heymann, Benjamin and Gilotte, Alexandre and Chan-Renous, R{\'e}mi},
  journal={arXiv preprint arXiv:2407.10504},
  year={2024}
}

@article{anytime_bound,
author = {Steven R. Howard and Aaditya Ramdas and Jon McAuliffe and Jasjeet Sekhon},
title = {{Time-uniform, nonparametric, nonasymptotic confidence sequences}},
volume = {49},
journal = {The Annals of Statistics},
number = {2},
publisher = {Institute of Mathematical Statistics},
pages = {1055 -- 1080},
year = {2021},
doi = {10.1214/20-AOS1991},
URL = {https://doi.org/10.1214/20-AOS1991}
}

@inproceedings{lin_bandit,
 author = {Abbasi-yadkori, Yasin and P\'{a}l, D\'{a}vid and Szepesv\'{a}ri, Csaba},
 booktitle = {Advances in Neural Information Processing Systems},
 editor = {J. Shawe-Taylor and R. Zemel and P. Bartlett and F. Pereira and K.Q. Weinberger},
 pages = {},
 publisher = {Curran Associates, Inc.},
 title = {Improved Algorithms for Linear Stochastic Bandits},
 volume = {24},
 year = {2011}
}

@book{conv_opt,
  author = {Boyd, Stephen and Vandenberghe, Lieven},
  description = {CiteULike},
  howpublished = {Hardcover},
  isbn = {0521833787},
  month = {March},
  publisher = {{Cambridge University Press}},
  title = {Convex Optimization},
  year = 2004
}

@book{concentration_ineq,
  author       = {St{\'{e}}phane Boucheron and
                  G{\'{a}}bor Lugosi and
                  Pascal Massart},
  title        = {Concentration Inequalities - {A} Nonasymptotic Theory of Independence},
  publisher    = {Oxford University Press},
  year         = {2013},
  url          = {https://doi.org/10.1093/acprof:oso/9780199535255.001.0001},
  doi          = {10.1093/ACPROF:OSO/9780199535255.001.0001},
  isbn         = {978-0-19-953525-5},
  bibsource    = {dblp computer science bibliography, https://dblp.org}
}

@inproceedings{betlei2024maximizing,
  title={Maximizing the success probability of policy allocations in online systems},
  author={Betlei, Artem and Vladimirova, Mariia and Sebbar, Mehdi and Urien, Nicolas and Rahier, Thibaud and Heymann, Benjamin},
  booktitle={Proceedings of the AAAI Conference on Artificial Intelligence},
  volume={38},
  number={10},
  pages={11061--11068},
  year={2024}
}

@book{krishna2009auction,
  title={Auction theory},
  author={Krishna, Vijay},
  year={2009},
  publisher={Academic press}
}

@InProceedings{heymann2023repeated,
author="Heymann, Benjamin
and Gilotte, Alexandre
and Chan-Renous, R{\'e}mi",
editor="Mavronicolas, Marios
and Qi, Qi
and Schoenebeck, Grant",
title="Repeated Bidding with Dynamic Value",
booktitle="Web and Internet Economics",
year="2026",
publisher="Springer Nature Switzerland",
address="Cham",
pages="548--563",
isbn="978-3-032-08560-3"
}

@misc{tropp2011,
      title={Freedman's inequality for matrix martingales}, 
      author={Joel A. Tropp},
      year={2011},
      eprint={1101.3039},
      archivePrefix={arXiv},
      primaryClass={math.PR},
      url={https://arxiv.org/abs/1101.3039}, 
}

@inproceedings{10.1145/3447548.3467280,
author = {Bompaire, Martin and Gilotte, Alexandre and Heymann, Benjamin},
title = {Causal Models for Real Time Bidding with Repeated User Interactions},
year = {2021},
isbn = {9781450383325},
publisher = {Association for Computing Machinery},
address = {New York, NY, USA},
url = {https://doi.org/10.1145/3447548.3467280},
doi = {10.1145/3447548.3467280},
booktitle = {Proceedings of the 27th ACM SIGKDD Conference on Knowledge Discovery \& Data Mining},
pages = {75–85},
numpages = {11},
location = {Virtual Event, Singapore},
series = {KDD '21}
}

@article{bompaire2024fixed,
  title={Fixed point label attribution for real-time bidding},
  author={Bompaire, Martin and D{\'e}sir, Antoine and Heymann, Benjamin},
  journal={Manufacturing \& Service Operations Management},
  volume={26},
  number={3},
  pages={1043--1061},
  year={2024},
  publisher={INFORMS}
}

@article{heymann2025non,
  title={Non-Linear Counterfactual Aggregate Optimization},
  author={Heymann, Benjamin and Sakhi, Otmane},
  journal={arXiv preprint arXiv:2509.03438},
  year={2025},
note = {Presented at Recsys 2025 Consequences workshop}
}

@inproceedings{achddou2021efficient,
  title={Efficient algorithms for stochastic repeated second-price auctions},
  author={Achddou, Juliette and Capp{\'e}, Olivier and Garivier, Aur{\'e}lien},
  booktitle={Algorithmic Learning Theory},
  pages={99--150},
  year={2021},
  organization={PMLR}
}

@article{balseiro2019learning,
  title={Learning in repeated auctions with budgets: Regret minimization and equilibrium},
  author={Balseiro, Santiago R and Gur, Yonatan},
  journal={Management Science},
  volume={65},
  number={9},
  pages={3952--3968},
  year={2019},
  publisher={INFORMS}
}

@incollection{diemert2017attribution,
  title={Attribution modeling increases efficiency of bidding in display advertising},
  author={Diemert, Eustache and Meynet, Julien and Galland, Pierre and Lefortier, Damien},
  booktitle={Proceedings of the ADKDD'17},
  pages={1--6},
  year={2017}
}

@inproceedings{weed2016online,
  title={Online learning in repeated auctions},
  author={Weed, Jonathan and Perchet, Vianney and Rigollet, Philippe},
  booktitle={Conference on Learning Theory},
  pages={1562--1583},
  year={2016},
  organization={PMLR}
}

@inproceedings{balseiro2020dual,
  title={Dual mirror descent for online allocation problems},
  author={Balseiro, Santiago and Lu, Haihao and Mirrokni, Vahab},
  booktitle={International Conference on Machine Learning},
  pages={613--628},
  year={2020},
  organization={PMLR}
}

@article{castiglioni2022unifying,
  title={A unifying framework for online optimization with long-term constraints},
  author={Castiglioni, Matteo and Celli, Andrea and Marchesi, Alberto and Romano, Giulia and Gatti, Nicola},
  journal={Advances in Neural Information Processing Systems},
  volume={35},
  pages={33589--33602},
  year={2022}
}

@article{guo2023reinforcement,
  title={Reinforcement learning for linear-convex models with jumps via stability analysis of feedback controls},
  author={Guo, Xin and Hu, Anran and Zhang, Yufei},
  journal={SIAM Journal on Control and Optimization},
  volume={61},
  number={2},
  pages={755--787},
  year={2023},
  publisher={SIAM}
}

@article{wang2020reinforcement,
  title={Reinforcement learning in continuous time and space: A stochastic control approach},
  author={Wang, Haoran and Zariphopoulou, Thaleia and Zhou, Xun Yu},
  journal={Journal of Machine Learning Research},
  volume={21},
  number={198},
  pages={1--34},
  year={2020}
}

@article{badanidiyuru2022incrementality,
  title={Incrementality bidding via reinforcement learning under mixed and delayed rewards},
  author={Badanidiyuru Varadaraja, Ashwinkumar and Feng, Zhe and Li, Tianxi and Xu, Haifeng},
  journal={Advances in Neural Information Processing Systems},
  volume={35},
  pages={2142--2153},
  year={2022}
}

@article{lewis2022incrementality,
  title={Incrementality bidding and attribution},
  author={Lewis, Randall and Wong, Jeffrey},
  journal={arXiv preprint arXiv:2208.12809},
  year={2022}
}

@inproceedings{yildiz2021continuous,
  title={Continuous-time model-based reinforcement learning},
  author={Yildiz, Cagatay and Heinonen, Markus and L{\"a}hdesm{\"a}ki, Harri},
  booktitle={International Conference on Machine Learning},
  pages={12009--12018},
  year={2021},
  organization={PMLR}
}

@article{darmasubramanian2025ads,
  title={Ads that Stick: Near-Optimal Ad Optimization through Psychological Behavior Models},
  author={Darmasubramanian, Kailash Gopal and Pareek, Akash and Khan, Arindam and Agarwal, Arpit},
  journal={arXiv preprint arXiv:2509.20304},
  year={2025}
}

@book{hartman2002ordinary,
  title={Ordinary differential equations},
  author={Hartman, Philip},
  year={2002},
  publisher={SIAM}
}

@book{lattimore2020bandit,
  title={Bandit algorithms},
  author={Lattimore, Tor and Szepesv{\'a}ri, Csaba},
  year={2020},
  publisher={Cambridge University Press}
}

\newpage
\appendix

\section*{Appendix Outline}
\label{appendix}
\begin{itemize}
    \item[A]\hyperref[appendix:concentration-lemmata]{Concentration Lemmata}
    \item[B]\hyperref[appendix-theoretical-solution]{Proofs on the Theoretical Solution}
    \item[C]\hyperref[appendix-K-concentration]{Proofs and remarks for Section~\ref{subsubsection:k} (concentration of $K$)}
    \item[D]\hyperref[appendix-Q-concentration]{Proofs and remarks for Section~\ref{subsubsection:q} (concentration of $Q$)}
    \item[E]\hyperref[appendix-algo]{Analysis of the algorithms}
    \item[F]\hyperref[appendix-lowerbound]{Analysis of the lower bound}
\end{itemize}

\section{Concentration Lemmata}
\label{appendix:concentration-lemmata}
In this section, we give the different lemmas that will be essential to prove our regret bounds. We begin by classical concentration inequalities.

For any square matrix $A \in \mathcal{M}_n(\mathbb{R})$ of size $n$, we denote by $\lambda_{\min}(A)$, the lowest eigenvalue of $A$.
We also define the operator norm of a matrix $\lVert\cdot\rVert_{\textsc{OP}}$ to be the operator norm of a matrix, 

\begin{tcolorbox}
\begin{lemma}{Concentration of OLS under sub-Gaussian Martingale noise.}\label{lemma:ols}

Let $\{\mathcal{F}_k\}_{k=0}^n$ be a filtration. For $0 \le k \le n$, we observe:
\begin{align*}
    y_k = x_k^\top a^\star + \epsilon_k\,,
\end{align*}
with $x_k \in \mathbb{R}^I$ is $\mathcal{F}_{k-1}$ measurable and $\epsilon_k$ is conditionally $\sigma_0$ sub-gaussian.

Let $X \in \mathbb{R}^{n \times I}$ be the design matrix (with rows $x_t^\top$). Define the regularized OLS estimator of $a^\star$ as:
\begin{align*}
    \hat{a} =  (X^\top X + \lambda_0 \mathcal{I})^{-1}X^\top y\,,
\end{align*}
with $\lambda_0 > 0.$ Then for any $\delta \in (0, 1)$, with probability at least $1 - \delta$, we have:
\begin{align*}
    \max_{i \in [I]}\lvert \hat{a}_i - a^\star_i\rvert \le \sigma_0 \sqrt{2\frac{\log(I/\delta)}{\lambda_{\min}(X^\top X)+ \lambda_0}}\,.
\end{align*}
\end{lemma}
\end{tcolorbox}

This lemma follows directly from a (coordinate wise) specialization of the self-normalized martingale concentration inequality of \cite{lin_bandit}.

\begin{tcolorbox}
\begin{lemma}{Vector Hoeffding Azuma for Martingales.}\label{lemma:hoeffding_azuma}

Let $\{V_k\}_{k \in [n]}$ be a sequence of vectors adapted to a filtration $\mathcal{F}_k$. Assume the vectors are conditionally bounded ($||V_k||_2 \le G$ almost surely, given $\mathcal{F}_{k-1}$) and of null conditional mean ($\mathbb{E}\left[V_k|\mathcal{F}_{k-1}\right] = 0$) for all $k \in [n]$. 

We have the following result with probability $1 - \delta$:
\begin{align*}
    \left\lVert \sum_{k \in [n]} V_k \right\rVert_2 \le G \sqrt{2 n \log 2/\delta}
\end{align*}    
\end{lemma}
\end{tcolorbox}

This lemma is the classical Hoeffding Azuma result on scalar martingales extended to vector martingales \cite{concentration_ineq}.

\begin{tcolorbox}
\begin{lemma}{Matrix Chernoff Bound for Martingales.} \label{lemma:matrix_chernoff}

Let $\{X_k\}_{k \in [n]}$ be a sequence of $d \times d$ self adjoint matrices adapted to a filtration $\mathcal{F}_k$. Assume the matrices are conditionally PSD and bounded:
\begin{align*}
    0 \le X_k \le R \cdot \mathcal{I}\,, \quad \text{almost surely, given } \mathcal{F}_{k-1}
\end{align*}
Define the sum of conditional expectations:
\begin{align*}
    M_n = \sum_{k = 1}^n \mathbb{E}[X_k| \mathcal{F}_{k-1}]
\end{align*}
For any $\epsilon \in [0, 1)$ and a fixed target $\nu$, we have: 
\begin{align*}
    \mathbb{P}\left( \lambda_{\min}\left(\sum_{k=1}^n X_k\right) \le (1 - \epsilon) \nu \quad \text{ and } \quad \lambda_{\min}(M_n) \ge \nu \right) \le d \cdot \exp\left(- \frac{\epsilon^2\nu}{2R} \right)
\end{align*}

\end{lemma}
\end{tcolorbox}

To obtain this lemma, one can use the Bernstein Matrix inequality from \cite{tropp2011} and simple matrix manipulations to link the variance to the mean. 

Finally, we give a matrix lemma that will be useful to link the minimal eigenvalue of PSD matrices.

\begin{tcolorbox}
\begin{lemma}{Weyl's inequality.}\label{lemma:weyl}

For two self-adjoint (thus Positive Semi Definite) matrices $A$ and $B$, we have:
\begin{align*}
    \lambda_{\min}(A) \ge \lambda_{\min}(B) - ||A - B||_{\textsc{op}}\,.
\end{align*}
\end{lemma}
\end{tcolorbox}

\section{Proofs on the Theoretical Solution}
\label{appendix-theoretical-solution}
\subsection{Proof of Theorem~\ref{lemma:cauchy}}
\label{proof:lemma:cauchy}
\begin{proof}
This is a direct consequence from combining Lemma 3 and Lemma 5 from~\cite{heymann2023repeated}.
\end{proof}
\subsection{Proof of Lemma~\ref{lemma:b_dynamics}}
\label{proof:lemma:b_dynamics}
\begin{proof}
This is a replication of the proof of Lemma 4 from~\cite{heymann2023repeated} wherever $k$ admits a derivative. 
\end{proof}
\subsection{Properties  of $\beta$}
The next simple fact is used implicitly in many places throughout the appendix.
\begin{tcolorbox}
   \begin{lemma}
\label{lemma:remark-on-beta}
  For any $\pi\in\mathcal{P}$,  $\beta(\pi) = \pi(t_I)$.  
\end{lemma} 
\end{tcolorbox}
\begin{proof}
  
    By Lemma~\ref{lemma:b_dynamics},
    $
      \dot{\pi}(t) = \dot{k}(t)-\gamma(V^\star(0)+k(t)) +\mu f(\pi(t))+\gamma \pi(t)
$, therefore, for $t>t_I$, 
 $
      \dot{\pi}(t) =\underbrace{\mu f(\pi(t))+\gamma \pi(t)}_{\text{increasing in $\pi$}}-\underbrace{\gamma(V^\star(0)+k(t_I))}_{\text{constant}} 
$, and since $\pi$ is bounded by 1  and non-decreasing, we conclude that $\dot{\pi}(t)=0$ for $t>t_I$.   Remembering that $\beta(\pi)=\lim_{t\to + \infty}\pi(t)$, the result follows. 
    
\end{proof}

\begin{tcolorbox}
\begin{lemma}[Monotonicity of $\beta$ w.r.t. $k$]
\label{lemma-monotony-beta-k}
Let $k_1,k_2\in \mathcal{K}$, such that $k_1\leq k_2$ and 
$q\in \mathcal{Q}$ then $S(k_1,q)\leq S(k_2,q)$. 
\end{lemma}
\end{tcolorbox}

\begin{proof}
  By Lemma~\ref{lemma:b_dynamics}, with obvious notations,
    \begin{align*}
         \underbrace{\gamma(V^\star_i(0)+k_i(\infty))}_{\text{increasing in $k$, }} = \underbrace{ \mu f(\pi_i(\infty))+\gamma \pi_i(\infty)}_{\text{increasing in $\pi$}},  
    \end{align*}
\end{proof}

\begin{tcolorbox}
\begin{lemma}[Monotonicity of $\beta$ w.r.t. $q$]
\label{lemma-monotony-beta-q}
Let \( k\in \mathcal{K}\) and \(q_1,q_2\in \mathcal{Q}\) 
such that 
\(q_2\leq q_1\), then 
\begin{align*}
    \beta(\mathcal{S}(k,q_2))\geq \beta(\mathcal{S}(k,q_1)).
\end{align*}
\end{lemma}
\end{tcolorbox}

In the proof, we use the shorthands $\pi_i =\mathcal{S}(k,q_i) $, $\beta_i = \beta(\pi_i)$, $f_i(t)=\int_0^t q_i(s)\dd s$.
\begin{proof}
Suppose the contrary,
then the set $\{t, \pi_2(s)\leq \pi_1(s) \quad \forall  s\in[t, T]\}$ is \textbf{non empty} (contains $T$ by assumption) and \textbf{bounded} ($\pi_1(0) = \pi_2(0)=k(t)$), and \textbf{closed} (the $\pi_i$ are continuous). It therefore \textbf{admit a minimum} $\tau$ at which $\pi_1(\tau)=\pi_2(\tau):=z$ (by continuity of $\pi_i$).
 By Lemma~\ref{lemma:b_dynamics}, 
     \begin{align*}
      \label{eq:plop}
          \dot{\pi}_i(\tau)&=\dot{k}(\tau)+\mu \Bigg( f_i\left(\pi(\tau)\right)-f_i\left(\beta_i\right)\Bigg)  +\gamma \Bigg(\pi_i(\tau) - \beta_i \Bigg)  -\gamma \Bigg(k(\tau)-k(T)\Bigg) \\
          &=\dot{k}(\tau)-\mu \Bigg( \int_{\pi_i(\tau)}^{\beta_i}q_i(s) \dd s\Bigg)  +\gamma \Bigg(\pi_i(\tau) - \beta_i \Bigg)  -\gamma \Bigg(k(\tau)-k(T)\Bigg) ,
      \end{align*}
therefore
\begin{align}
    \dot{\pi}_1(\tau)-\dot{\pi}_2(\tau)) = \underbrace{\mu \Bigg( \int_{z}^{\beta_2}q_2(s) \dd s-\int_{z}^{\beta_1}q_1(s) \dd s\Bigg)}_{(\beta_2<\beta_1, q_2\leq q_1 ) \implies \leq 0  }  +\underbrace{\gamma \Bigg(\Big[\pi_1(\tau) - \beta_1\Big]-\Big[\pi_2(\tau) - \beta_2 \Big] \Bigg)}_{=\beta_2 - \beta_1<0}  
\end{align}

Therefore $\dot{\pi}_1(\tau)<\dot{\pi}_2(\tau)$ and 
$\pi_1(\tau)=\pi_2(\tau)$ 
and 
which means that $\pi_2(s)< \pi_1(s)$ in a left neighborhood of $\tau$, which yields a contradiction with the definition of $\tau$.
\end{proof}

For $k\in \mathcal{K}$, $q\in\mathcal{Q}$, $\pi\in\mathcal{P}$ and $t\in \mathbb{R}_+$, we denote by $\mathcal{U}(k,q,\pi,t)$ the expected payoff of policy $\pi$ when evaluated in the environment $(k,q)$ and starting at time $t$. The function $\mathcal{U}$ enjoys  the following regularity property.

\textit{(Lemma~\ref{lemma:dynamicsensitivity} 
is used in the proof of Lemma~\ref{lemma:epsilon-opt}.)}

\begin{tcolorbox}
    \begin{lemma}
\label{lemma:dynamicsensitivity}
Let $k$ and $\hat{k}$ in $\mathcal{K}$, $q$ and $\hat{q}$ in $\mathcal{Q}$   such that for some $\epsilon>0$,
 $\mid\mid k-\hat{k}\mid \mid_{\infty}\leq \epsilon $  and $\max_{v\in[0,\beta(\mathcal{S}(k,q))]}\mid q(v)-\hat{q}(v)\mid\leq \epsilon $.
Then
    \begin{align*}
        \mid\mid \mathcal{U}(k,q,\mathcal{S}(k,q),.) -  \mathcal{U}(\hat{k},\hat{q},\mathcal{S}(k,q),.) \mid\mid_{\infty}\leq \frac{4\mu^2}{\gamma^2} \epsilon.
    \end{align*}
\end{lemma}
\end{tcolorbox}

\begin{proof}
Set  $\pi= S(k,q)$.
    Using the abuse of notation $q(t)=q(\pi(t))$ and $\hat{q}(t)=\hat{q}(\pi(t))$... we have  by dynamic programming (see \cite{heymann2023repeated}) the identity 
\begin{align*}
&\mathcal{U}(k,q,\pi,\tau)
=\\
&\int_0^{+\infty} e^{-(\gamma+\mu) t} \cdot \mu \cdot \Bigg(\mathcal{U}(k,q,\pi,t+\tau)\cdot\bigg(1-q(t+\tau)\bigg) \\ &+q(t+\tau)\cdot \bigg( k(t+\tau)+\mathcal{U}(k,q,\pi,0)\bigg) - p(t+\tau)\Bigg) \dd t
\end{align*}
Therefore, 
  \begin{align*}
      &\underbrace{\mid\mid  \mathcal{U}(k,q,\pi,\tau) -  \mathcal{U}(\hat{k},\hat{q},\pi,\tau) \mid\mid_{\infty}}_{A}=\\&
      \int_0^{+\infty} e^{-(\gamma+\mu) t} \cdot \mu \cdot \Bigg(
      \underbrace{\mathcal{U}(k,q,\pi,t+\tau)\cdot \bigg(1-q(t+\tau)\bigg) - U(\hat{k},\hat{q},\pi,t+\tau)\cdot \bigg( (1-\hat{q}(t+\tau)\bigg)}_{B} \\
      &+ \underbrace{q(t+\tau)\mathcal{U}(k,q,\pi,0) - \hat{q}(t+\tau)\mathcal{U}(\hat{k},\hat{q},\pi,0)}_{C}\\&
      +\underbrace{q(t+\tau)k(t+\tau) - p(t+\tau) 
      -\hat{q}(t+\tau)\hat{k}(t+\tau) + \hat{p}(t+\tau)}_{D} \Bigg)\dd t 
      \end{align*}
      Now, we have

      \begin{align*}
          B \leq A(1-q(t+\tau)) +  
          \frac{\mu}{\gamma}\cdot\epsilon, \\
               C \leq A q(t+\tau)  + \frac{\mu}{\gamma}\cdot\epsilon, \\
                    D \leq 3 \epsilon,
      \end{align*}
      therefore, 
      \begin{align*}
          A \leq \frac{\mu}{\mu+\gamma}\big(A +(\frac{2\mu}{\gamma}+3)\epsilon \big),
      \end{align*}
 which implies
 \begin{align*}
    A \frac{\gamma}{\mu+\gamma}\leq \frac{\mu}{\mu+\gamma} \frac{2\mu+3\gamma}{\gamma}\epsilon .
\end{align*}
Hence
\begin{align*}
    A \leq 4\frac{\mu^2}{\gamma^2}\epsilon.
\end{align*}
\end{proof}

\subsection{Proof of Theorem~\ref{th:extension-theorem}}
\label{proof:extension-theorem}
\begin{proof}
    Denote by $\hat{V}$ (resp. $\tilde{V}$) the Bellman value associated with $\hat{q}$ (resp. $\tilde{q}$).
    It was shown in Lemma 2 from~\cite{heymann2023repeated} that $\hat{\pi}=\mathcal{S}(k,\hat{q})$ satisfies 
    \begin{align*}
        \hat{\pi}(t) = k(t) + \hat{V}(0) - \hat{V}(t).
    \end{align*}
    This implies in particular,  the relation
    \begin{align*}
        k(t) + \hat{V}(0) - \hat{V}(t) \leq \max_{t\geq 0 } \hat{\pi}(t) =  \beta(\mathcal{S}(k,\hat{q})),
    \end{align*}
    and so, using the lemma assumption, for all $t\in[0,+\infty]$,
    \begin{align*}
        \hat{f}(k(t) + \hat{V}(0) - \hat{V}(t)) =  \tilde{f}(k(t) + \hat{V}(0) - \hat{V}(t)), 
    \end{align*}
    where 
    \begin{align*}
        \hat{f}(v) = \hat{q}(v)v - \hat{p}(v)\\
        \tilde{f}(v) = \tilde{q}(v)v - \tilde{p}(v)
    \end{align*}

From there we deduce by uniqueness of the solution of the Cauchy problem~\eqref{eq:cauchy} for $y_0=\hat{V}(0)$, that $\tilde{V}(t) = \hat{V}(t)$ for all $t$.
Therefore, using Lemma~\ref{lemma:b_dynamics}, 
\begin{align*}
     \mathcal{S}(k,\hat{q})
      = 
   \mathcal{S}(k,\tilde{q}).
\end{align*}
The conclusion follows from the fact that $\mathcal{S}(k,\tilde{q})(t) \leq \beta(S(k,\hat{q}))$ for all $t$, which implies by hypothesis that $ \mathcal{U}(k,\hat{q},\mathcal{S}(k,\hat{q}),0) 
      = 
    \mathcal{U}(k,\tilde{q},\mathcal{S}(k,\tilde{q}),0)$.
    \end{proof}
Lemma~\ref{lemma:q_extension} will be used in the proof of Lemma~\ref{lemma:epsilon-opt}.

\begin{tcolorbox}
\begin{lemma}[$q$ extension]
\label{lemma:q_extension}

For any $\varepsilon>0$, $\beta>0$ and $(\tilde{q},q)\in\mathcal{Q}$ such that $\mid\tilde{q}(b)-q(b)\mid\leq \varepsilon$ for any $b\in[0,\beta]$, there exists a continuous $q_{ext}$ such that 
\begin{align*}
    \tilde{q}(b)=q_{ext}(b)  \quad   \forall b\in[0,\beta],\\
    \mid q_{ext}(b)-q(b)\mid\leq \varepsilon
    \quad   \forall b\in\mathbb{R}_+ .
\end{align*}
\end{lemma}
\end{tcolorbox}

\begin{proof}
   Set  
   \begin{align*}
   q_{ext}(b) =  \tilde{q}(b) \cdot [b<\beta] +(\tilde{q}(\beta) -q(\beta)  + q(b))\cdot[b\geq\beta].
\end{align*}
\end{proof}
Lemma~\ref{lemma:epsilon-opt} formalize the crucial observation that to produce an $\epsilon$-optimal policy, it is sufficient to have a good estimator $\hat{q}$ of $q$ up to $\beta(\mathcal{S}(k,\hat{q}))$. This last quantity  does not depend explicitly on $q$. The Lemma is used to prove Theorem~\ref{th:asymp}. 

\begin{tcolorbox}
\begin{lemma}
\label{lemma:epsilon-opt}
Let $(k,\hat{k})\in\mathcal{K}^2$ and  $(q,\hat{q})\in\mathcal{Q}^2$ such that $\mid\mid k-\hat{k}\mid\mid_\infty\leq \epsilon$ and $\max_{v\in[0,\beta(\mathcal{S}(\hat{k},\hat{q}))]}\mid q(v)-\hat{q}(v)\mid\leq \epsilon $, then
    \begin{align*}
        \mathcal{U}(k,q,\mathcal{S}(\hat{k},\hat{q}),0) \geq        \mathcal{U}(k,q,S(k,q),0)  -\frac{4\mu^2}{\gamma^2}\varepsilon.
    \end{align*}
\end{lemma}
\end{tcolorbox}

\begin{proof}
By Lemma~\ref{lemma:q_extension}, there exists $\tilde{q}\in\mathcal{Q}$ such that 
\begin{align*}
    \tilde{q}(b) = \hat{q}(b) \quad \forall b\in [0,\beta(\mathcal{S}(\hat{k},\hat{q}))]\quad \text{and}\quad \mid\mid\tilde{q} - q\mid\mid_\infty \leq \epsilon.
\end{align*}
Then we have 
\begin{align*}
     \mathcal{U}(\hat{k},\hat{q},\mathcal{S}(\hat{k},\hat{q}),0) \\
     \intertext{by Theorem~\ref{th:extension-theorem}}
      = 
    \mathcal{U}(\hat{k},\tilde{q},\mathcal{S}(\hat{k},\tilde{q}),0)\\
    \intertext{by definition of $\mathcal{S}(\hat{k},\tilde{q})$}
     \geq \mathcal{U}(\hat{k},\tilde{q},S(k,q),0)\\
     \intertext{by Lemma~\ref{lemma:dynamicsensitivity}}
     \geq \mathcal{U}(k,q,S(k,q),0) - 4\frac{\mu^2}{\gamma^2}\epsilon.
\end{align*}
\end{proof}

\subsection{Proof of Theorem~\ref{th:asymp}}
\label{proof:th:asymp}
\begin{proof}
Take the $\lim\sup$ of $B=\beta(\pi_n)$. By construction for $n$ large enough, we are have an arbitrary good approximation of $q^\star$ on $[0,B]$ and of $k^\star$. 
The conclusion follows from Lemma~\ref{lemma:epsilon-opt}.
\end{proof}

\begin{tcolorbox}
\begin{lemma}\label{lemma:value_bregman}
For any policy $\pi$
    \begin{align}
    ||V^\star_{.}-V^{\pi}_{.}||_{\infty}\leq \frac{\mu}{\gamma} ||D_{Q}(\pi^\star_{.};\pi_{.})||_{\infty}
\end{align}
 where $D_Q$ is the Bregman divergence of $Q:b\to \int_0^b q(b) \dd b $.
\end{lemma}
\end{tcolorbox}

We use the subscript notation in the proof for better readability.
\begin{proof}
    Let $t>0$, by the dynamic programming relation (see~\cite{heymann2023repeated}), 
    $V^{\star}_t - V^\pi_t$ is equal to 
    \begin{align}
    \label{eq:D61}
       \mathbb{E}_\tau \Bigg[e^{-\gamma \tau}\Big(
    V^\star_{t+\tau} + U(k_{t+\tau}+ V^\star_0-V^\star_{t+\tau},\pi^\star_{t+\tau})
     - V^\pi_{t+\tau} - U(k_{t+\tau}+ V^\pi_0-V^\pi_{t+\tau},\pi_{t+\tau})
    \Big)\Bigg]. 
    \end{align}
   After adding and removing the term $U\left(k_{t+\tau}+ V^\star_0-V^\star_{t+\tau},\pi_{t+\tau}\right)$, and observing that 
   \begin{align*}
      U\left(k_{t+\tau}+ V^\star_0-V^\star_{t+\tau},\pi_{t+\tau}\right)- U(k_{t+\tau}+ V^\pi_0-V^\pi_{t+\tau},\pi_{t+\tau})
    \Big)= q(\pi_{t+\tau})\Big( V^\star_0-V^{\pi}_0   +V^\pi_{t+\tau}-V^\star_{t+\tau}\Big) 
   \end{align*}
   we can rewrite ~\eqref{eq:D61} as 
   \begin{align*}
\mathbb{E}_\tau \Bigg[e^{-\gamma \tau}\Big(
     \underbrace{U(k_{t+\tau}+ V^\star_0-V^\star_{t+\tau},\pi^\star_{t+\tau})
    -U(k_{t+\tau}+ V^\star_{0}-V^\star_{t+\tau},\pi_{t+\tau})}_{D_Q(\pi_{t+\tau}^\star,\pi_{t+\tau})} \\+   ( V^\star_{t+\tau} - V^\pi_{t+\tau})+q(\pi_{t+\tau})\Big( V^\star_0-V^{\pi}_0   +V^\pi_{t+\tau}-V^\star_{t+\tau}\Big)
    \Big)\Bigg],
   \end{align*}
   which is upper bounded by 
   \begin{align*}
        \frac{\mu}{\mu+\gamma} \big(||D_{Q}(\pi^\star_{.};\pi_{.})||_{\infty}+ ||V^\star_{.}-V^{\pi}_{.}||_{\infty}\big).
   \end{align*}
   We hence have proved that
   \begin{align*}
      || V_{.}^\star- V_{.}^\pi||_{\infty}  \leq \frac{\mu}{\mu+\gamma} \big(||D_{Q}(\pi^\star_{.};\pi_{.})||_{\infty}+ ||V^\star_{.}-V^{\pi}_{.}||_{\infty}\big),
   \end{align*}
   from which we can derive the claim.
\end{proof}

\newpage
\section{Proofs and remarks for Section~\ref{subsubsection:k} (concentration of $K$)}
\label{appendix-K-concentration}

\subsection*{Overall Problem and Result}

The value $k$ of winning an auction is modeled by a concave, increasing piecewise linear function:
\begin{align*}
    \forall \tau \ge 0\,, \quad k(\tau) = \sum_{i = 1}^{I_k}a^\star_i \cdot (t_i \wedge \tau - t_{i-1} \wedge \tau )\,,
\end{align*}
with $a^\star$ lies in $C = \left\{ a \in \mathbb{R}^I \mid 0 = a_{I_k} \leq a_{I_k-1} \leq \cdots \leq a_1 \right\}$, $t_{I_{k}-1} = T_f$ and $t_{I_k} = +\infty$.

We add the following technical condition:
\begin{itemize}
    \item $0< \Delta^k_{\min} \le |t_{i} - t_{i-1}| \le \Delta^k_{\max}$ for all $i \in [I_k]$.
    \item The maximum number of auctions is $L \ge 1$ finite.
\end{itemize}

At the end of an episode $n$ with at least one auction won, we observe the different times we won the auctions $\{\tau^n_1, \cdots, \tau^n_{\ell}\}$ and the aggregated, final feedback $v_n$. In expectation over the aggregated feedback, we have:
\begin{align*}
    \mathbb{E}[v_n] = \sum_{j = 1}^\ell k(\tau^n_\ell) &= \sum_{i = 1}^{I_k} a^\star_i \cdot \sum_{j = 1}^\ell (t_i \wedge \tau^n_j - t_{i-1} \wedge \tau^n_j)  \\
    &= \left \langle a^{\star}\,,  \sum_{j = 1}^\ell z^j_n \right \rangle \\
    &= \left \langle a^{\star}\,,  z_n \right \rangle\,,
\end{align*}
with $z_n$ the vector $ \sum_{j = 1}^\ell z^j_n$ with $z^j_n = [ t_i \wedge \tau^n_j - t_{i-1} \wedge \tau^n_j]_{i \in [I_k]}$. 

This means that $a^\star$ can be estimated by a projected, regularized OLS.

In our setting, we suppose that we will interact with $N$ sessions. Let $N_e \le N$ be the number of sessions with at least one won auction ($\ell \ge 1$). This gives us a dataset $\mathcal{D}^k_{N_e}$ of the form: 
\begin{align*}
    \mathcal{D}^k_{N_e} = \left\{ z_n, v_n \right\}_{n \in [N_e]}.
\end{align*}

Let $Z$ be the design matrix and let \(\bar{a} = (Z^\top Z + \lambda_0 \mathcal{I})^{-1} Z^\top V\) be the regularized ordinary least squares estimator of \( a^\star \). 

We then define the projected estimator:
\[
\hat{a} = \textsc{proj}_C(\bar{a}),
\]
where \(\textsc{proj}_C\) denotes the projection onto the convex set \( C \). Finally, our estimated value $\hat{k}$ is then given by:
\begin{align*}
    \forall \tau \ge 0\,, \quad \hat{k}_N(\tau) = \sum_{i = 1}^{I_k}\hat{a}_i \cdot (t_i \wedge \tau - t_{i-1} \wedge \tau )\,.
\end{align*}
In this section, we show that:
\begin{tcolorbox}
Let $\{\pi_n\}_{n \in [N]}$ be the bidders sequence, with each $\pi_n$ allowed to depend on the previously collected data $\mathcal{D}^k_{n}$.
\begin{condition}\label{cond:k}
There exists $\epsilon$ s.t. $\pi_n(t_1) \ge \epsilon$ for all $n \in [N]$.
\end{condition}
If Condition~\ref{cond:k} holds, for $N$ large enough, we have with high probability:
\begin{align*}
    \max_{\tau \ge 0}\,\lvert k(\tau) - \ \hat{k}_N(\tau)\rvert \le \mathcal{O}\left(\sqrt{1/N}\right)\,.
\end{align*}
\end{tcolorbox}
To prove this result, we follow the recipe below:
\begin{itemize}
    \item Lemma~\ref{app:link_N_e}: We link $N_e$ the number of episodes with at least one auction won to the number of episodes $N$.
    \item Lemma~\ref{app:lemma_learnability} shows that the uniform lower bound on the bidding quantity $\epsilon$ at $t_1$ allows learning.
    \item Proposition~\ref{prop:concentration_k} combines the above lemmas with OLS concentration to prove the result.
\end{itemize}

We suppose that Condition~\ref{cond:k} holds in the rest of the proofs.




\subsection*{Lemma to explicit dependency on the number of episodes}
\begin{tcolorbox}
\begin{lemma} \label{app:link_N_e}
    Let $N_e$ the number of episodes with at least one won auction and $\delta \in (0,1)$.
    
    Let $p_w = \frac{\mu \alpha \epsilon}{\mu \alpha \epsilon + \gamma} \cdot \exp(-\gamma t_1) > 0$. If $N \ge \frac{8 \log(1/\delta)}{p_w}$, we have with probability $1-\delta$: 
    \begin{align*}
        N_e \ge \frac{1}{2}N p_w\,.
        \end{align*}
\end{lemma}
\end{tcolorbox}
\begin{proof}
For each episode $n\in\{1,\dots,N\}$ define
\[
Y_n \;=\; \mathbf{1}\{\text{episode $n$ contains at least one won auction}\}
      \;=\;\mathbf{1}\{\tau_n<\infty\},
\qquad\text{so that}\qquad
N_e=\sum_{n=1}^N Y_n.
\]
Let $\mathcal{F}_{n}$ the natural filtration generated by the history up to the end
of episode $n$ (in particular $\pi_n$ is $\mathcal{F}_{n-1}$-measurable). Define
the conditional success probabilities
\[
p_n \;:=\; \mathbb{P}(Y_n=1\mid \mathcal{F}_{n-1})
          \;=\;\mathbb{P}_{n-1}(\tau_n<\infty),
\]
Let $S_N := \sum_{n=1}^N D_n = N_e - \sum_{n=1}^N p_n$. By Azuma-Hoeffding (Lemma~\ref{lemma:hoeffding_azuma}),
for any $\delta \in (0, 1)$,
\[
\mathbb{P}\!\left(
N_e \;\ge\; \sum_{n=1}^N p_n - \sqrt{2N\log(1/\delta)}
\right)\;\ge\;1-\delta.
\]
In addition, we have:
\begin{align*}
    \mathbb{P}_{n-1}(\tau_n < \infty) &= 1 - \mathbb{E}_{T \sim \text{Exp}(\gamma)} \left[ \exp\left(-\int_{0}^T \mu q(\pi_n(t)) dt \right )\right] \\
    &\ge 1 - \mathbb{E}_{T \sim \text{Exp}(\gamma)} \left[ \exp\left(-\mu \alpha \int_{0}^T  \pi_n(t) dt \right )\right] \quad (\text{$q(b) > \alpha b$ })\\
    &\ge 1 - \mathbb{E}_{T \sim \text{Exp}(\gamma)} \left[ \exp\left(-\mu \alpha \epsilon (T - t_1)_{+}  \right )\right] \quad (\pi_n(t_1) \ge \epsilon\,, \text{Condition~\ref{cond:k}}) \\
    &= \frac{\mu \alpha \epsilon}{\mu \alpha \epsilon + \gamma} \cdot \exp(-\gamma t_1) = p_{w}\,.
\end{align*}
Since $\sum_{n=1}^N p_n \ge N p_w$, this implies
\[
\mathbb{P}\!\left(
N_e \;\ge\; N p_w - \sqrt{2N\log(1/\delta)}
\right)\;\ge\;1-\delta.
\]
Under the condition $N \ge \frac{8\log(1/\delta)}{p_w}$ we finally obtain
\[
\mathbb{P}\!\left(N_e \ge \frac12 N p_w\right)\;\ge\;1-\delta\,,
\]
which proves our result. 
\end{proof}

\subsection*{Lemma to prove learnability}

As $\hat{k}$ is estimated by projected, regularized OLS and that our sequence of bidders can depend on the past, we invoke the concentration result of OLS under sub-Gaussian Martingale noise (Lemma~\ref{lemma:ols}) that ensures learnability if we can lower bound and control $\lambda_{\min}(Z_n^\top Z_n)$. 

To control this quantity, we invoke Matrix Chernoff Bound for Martingales (Lemma~\ref{lemma:matrix_chernoff}) as we have $Z_n^\top Z_n = \sum_{j = 1}^n z^\top_j z_j$.
\begin{tcolorbox}
\begin{lemma}\label{app:lemma_learnability}
Let $X_j = z^\top_j z_j$ for any $j \in [N_e]$, $L$ the maximum number of auctions,
\begin{align*}
    &p_c = \frac{\alpha\epsilon(\mu + \gamma)}{\gamma+\mu\alpha\epsilon}\, \exp\left(-(\mu+\gamma)T_f \right)\, > 0 \\
    &\nu = (\Delta^k_{\min})^2 p_c > 0 \\
    &R = (I_k \cdot L \cdot \Delta^k_{\max})^2\,.
\end{align*}
For any $\delta \in (0, 1)$, if $N_e \ge \frac{8R\log(I_k/\delta)}{\nu}$, we have with probability $1 - \delta$: 
\begin{align*}
    \lambda_{\min}\left(\sum_{k=1}^{N_e} X_k\right) \ge \frac{1}{2} N_e \nu \,.
\end{align*}
\end{lemma}
\end{tcolorbox}
\begin{proof}
We invoke Lemma~\ref{lemma:matrix_chernoff} with $X_j = z_j^\top z_j$. We have $R = \lVert X_j \rVert_{\textsc{op}} = \lVert z_j\rVert^2_2 = (I_k \cdot L \cdot \Delta^k_{\max})^2.$

The sum of conditional expectations is $M = \sum_{j = 1}^{N_e} \mathbb{E}\left[X_j \mid \mathcal{F}_{j-1}, \tau<\infty\right]$ as we are only looking at observations with at least one won auction. Conditioning on the filtration only fixes the bidding strategies used.
We have:
\begin{align*}
    \mathbb{E}\left[X_j \mid \mathcal{F}_{j-1}, \tau<\infty\right] &= \mathbb{E}_{j-1}\left[X_j \mid \tau<\infty\right] \\
    &= \mathbb{E}_{j-1}\left[\sum_{m, m'= 1}^\ell (z^{\tau_m}_j)^\top z^{\tau_{m'}}_j \mid \tau<\infty\right] \\
    &= \mathbb{E}_{j-1}[\ell\mid\tau<\infty] \mathbb{E}_{j-1}\left[ (z^{\tau}_j)^\top z^{\tau}_j \mid  \tau<\infty\right] \\ &+ \mathbb{E}_{j-1}[\ell(\ell-1)\mid\tau<\infty] \mathbb{E}_{j-1}\left[ z^\tau_j \mid \tau<\infty\right]^\top\mathbb{E}_{j-1}\left[ z^\tau_j \mid \tau<\infty\right]\,.
\end{align*}
Setting $S_j = \mathbb{E}_{j-1}\left[ (z^{\tau}_j)^\top z^{\tau}_j \mid  \tau<\infty\right]$, this means that:
\begin{align*}
    \lambda_{\min}\left(\mathbb{E}\left[X_j \mid \mathcal{F}_{j-1}, \tau<\infty\right] \right) &\ge \mathbb{E}_{j-1}[\ell\mid\tau<\infty] \lambda_{\min}\left(S_j\right)   \\
    &\ge  \lambda_{\min}\left(S_j \right)\,.
\end{align*}
Using the Rayleigh-Ritz characterization of the smallest eigenvalue, we have:
\begin{align*}
    \lambda_{\min}(S_j) = \min_{||v||= 1} v^tS_jv &= \min_{||v||= 1} \mathbb{E}_{j-1}\left[\sum_{i = 1}^{I_k} (v_i z_i)^2 \mid \tau < \infty\right].
\end{align*}
Conditioning on the interval in which $\tau$ falls and using the structure of $z$, we obtain
\[
\mathbb{E}_{j-1}\!\left[\sum_{i=1}^{I_k} (v_i z_i)^2 \,\middle|\, \tau<\infty\right]
= \sum_{m=1}^{I_k} p_m
   \sum_{i=1}^{m} (v_i \Delta_i)^2,
\]
where
\[
p_m := \mathbb{P}_{j-1}\!\left(\tau\in[t_m,t_{m+1}] \,\middle|\, \tau<\infty\right).
\]
Exchanging the order of summation yields
\[
\sum_{m=1}^{I_k} p_m \sum_{i=1}^{m} (v_i \Delta_i)^2
= \sum_{i=1}^{I_k} (v_i \Delta_i)^2 \sum_{m=i}^{I_k} p_m.
\]
Define the weights
\[
w_i := \sum_{m=i}^{I_k} p_m
      = \mathbb{P}_{j-1}\!\left(\tau \ge t_i \,\middle|\, \tau<\infty\right).
\]
Hence,
\[
\mathbb{E}_{j-1}\!\left[\sum_{i=1}^{I_k} (v_i z_i)^2 \,\middle|\, \tau<\infty\right]
= \sum_{i=1}^{I_k} w_i \Delta_i^2 v_i^2.
\]

Since the right-hand side is the quotient of the diagonal matrix
$\mathrm{diag}(w_i\Delta_i^2)$, we conclude that
\[
\lambda_{\min}(S_j)
= \min_{\|v\|_2=1} \sum_{i=1}^{I_k} w_i \Delta_i^2 v_i^2
= \min_{i\in[I_k]} \bigl(w_i \Delta_i^2\bigr).
\]
Now, let us lower bound $\mathbb{P}_{j-1}\!\left(\tau \ge t_i \,\middle|\, \tau<\infty\right)$, as $t_0 = 0$, this probability is equal to $1$ for $t_0$, so we look at $t_1$ and upwards:

We recall that
\[
\mathbb P_{j-1}(\tau<\infty)
= 1-\mathbb E_{T\sim\mathrm{Exp}(\gamma)}
\!\left[\exp\!\left(-\int_0^{T}\mu q(\pi_j(t))\,dt\right)\right].
\]
Moreover,
\begin{align*}
\mathbb P_{j-1}(\tau\ge t_i,\ \tau<\infty)
&= \mathbb E_{T}\!\Bigg[
\mathbf 1\{T\ge t_i\}
\exp\!\left(-\int_0^{t_i}\mu q(\pi_j(t))\,dt\right)
\Bigg(1-\exp\!\left(-\int_{t_i}^{T}\mu q(\pi_j(t))\,dt\right)\Bigg)
\Bigg].
\end{align*}
Therefore,
\[
\mathbb P_{j-1}(\tau \ge t_i \mid \tau<\infty)
=
\frac{
\mathbb E_{T}\!\left[
\mathbf 1\{T\ge t_i\}
e^{-\int_0^{t_i}\mu q(\pi_j(t))dt}
\left(1-e^{-\int_{t_i}^{T}\mu q(\pi_j(t))dt}\right)
\right]
}{
1-\mathbb E_{T}\!\left[e^{-\int_0^{T}\mu q(\pi_j(t))dt}\right]
}.
\]

For $t_i\ge t_1$, we lower bound the numerator by using the envelope bounds
$q(\pi_j(t))\le 1$ and $q(\pi_j(t))\ge \alpha\epsilon$ on $[t_i,T]$:
\[
e^{-\int_0^{t_i}\mu q(\pi_j(t))dt}\ \ge\ e^{-\mu t_i},
\qquad
1-e^{-\int_{t_i}^{T}\mu q(\pi_j(t))dt}\ \ge\ 1-e^{-\mu\alpha\epsilon (T-t_i)}.
\]
Hence
\[
\mathbb P_{j-1}(\tau\ge t_i,\tau<\infty)
\ge e^{-\mu t_i}\,
\mathbb E_T\!\left[\mathbf 1\{T\ge t_i\}\left(1-e^{-\mu\alpha\epsilon (T-t_i)}\right)\right].
\]
Using $\mathbb P(T\ge t_i)=e^{-\gamma t_i}$ and the memoryless property
$(T-t_i)\mid\{T\ge t_i\}\sim\mathrm{Exp}(\gamma)$ yields
\[
\mathbb E\!\left[1-e^{-\mu\alpha\epsilon (T-t_i)}\mid T\ge t_i\right]
=\frac{\mu\alpha\epsilon}{\gamma+\mu\alpha\epsilon},
\]
so that
\[
\mathbb P_{j-1}(\tau\ge t_i,\tau<\infty)
\ge e^{-(\mu+\gamma)T_f}\,\frac{\mu\alpha\epsilon}{\gamma+\mu\alpha\epsilon}.
\]
Finally, since $\mathbb P_{j-1}(\tau<\infty)\le \frac{\mu}{\gamma + \mu}$, dividing by
$\mathbb P_{j-1}(\tau<\infty)$ gives the claimed bound:
\[
\lambda_{\min}(S_j)
= \min_{i\in[I_k]} \bigl(w_i \Delta_i^2\bigr) \ge (\Delta^k_{\min})^2 p_c\,,
\]
with:
\begin{align*}
    p_c = \frac{\alpha\epsilon(\mu + \gamma)}{\gamma+\mu\alpha\epsilon}\, \exp\left(-(\mu+\gamma)T_f \right)\,.
\end{align*}
This means $\lambda_{\min}\left(M\right) \ge N_e \nu$, and finally, choosing $N_e \ge \frac{8R\log(I_k/\delta)}{\nu}$ ensures the following holds with probability $1-\delta$:
\begin{align*}
    \lambda_{\min}\left(\sum_{k=1}^{N_e} X_k\right) \ge \frac{1}{2} N_e \nu \,.
\end{align*}
\end{proof}

\subsection*{Final Concentration Proposition}
Now, let $\{\pi_n\}_{n \in [N]}$ be the bidders sequence, with each $\pi_n$ allowed to depend on the previously collected data $\mathcal{D}^k_{n}$. Suppose that there exists $\epsilon$ s.t. $\pi_n(t_1) \ge \epsilon$ for all $n \in [N]$.

For the final result, we combine three concentration results: Lemma~\ref{lemma:ols}, Lemma~\ref{app:link_N_e} and Lemma~\ref{app:lemma_learnability} to obtain:
\begin{tcolorbox}
\begin{proposition} \label{prop:concentration_k}
Let $\delta \in (0,1 )$ and $p_{\min} = p_c \, p_w = \mathcal{O}((\alpha\epsilon)^2).$
We also define 
$$N^k_0 = \max \left\{ \frac{16 R \log(3 I_k/\delta)}{(\Delta^k_{\min})^2p_{\min}}, \frac{8\log(3/\delta)}{p_w}\right\}$$
If $N \ge N^k_0$, then we have with probability $1-\delta$:
    \begin{align*}
    \max_{i \in [I]}\lvert \hat{a}_i - a^\star_i\rvert \le \max_{i \in [I]}\lvert \bar{a}_i - a^\star_i\rvert \le  \frac{2}{\Delta^k_{\min}}\sqrt{2L\frac{\log(3I_k/\delta)}{Np_{\min}}}\,,
\end{align*}
giving:
\begin{align*}
    \max_{\tau \ge 0} |k(\tau) - \hat{k}_N(\tau)|\le  \frac{2 T_f}{\Delta^k_{\min}}\sqrt{2L\frac{\log(3I_k/\delta)}{Np_{\min}}}\,.
\end{align*}
\end{proposition}
\end{tcolorbox}

\begin{proof}
With probability at least $1-\delta/3$, Lemma~\ref{app:link_N_e} ensures
$N_e \ge \tfrac12 N p_w$ as soon as $N \ge 8\log(3/\delta)/p_w$.
For learnability, we apply the matrix Chernoff martingale bound
(Lemma~\ref{app:lemma_learnability}) to
$Z_{N_e}^\top Z_{N_e}=\sum_{j=1}^{N_e} X_j$, which implies that, with probability
at least $1-\delta/3$,
\[
\lambda_{\min}(Z_{N_e}^\top Z_{N_e}) \ge \tfrac12 N_e \nu
\]
whenever $N_e \ge 8R\log(3I_k/\delta)/\nu$, where $\nu=(\Delta_{\min}^k)^2 p_c$.
On the intersection of these two events, we obtain
\[
\lambda_{\min}(Z_{N_e}^\top Z_{N_e})
\ge \tfrac14 N (\Delta_{\min}^k)^2 p_c p_w
= \tfrac14 N (\Delta_{\min}^k)^2 p_{\min}\,,
\]
as soon as $N \ge \max \left\{ \frac{16 R \log(3 I_k/\delta)}{(\Delta^k_{\min})^2p_{\min}}, \frac{8\log(3/\delta)}{p_w}\right\}$. 

The regression noise is a sum of at most $L$ conditionally $1$-sub-Gaussian terms
and is therefore conditionally $\sqrt{L}$-sub-Gaussian, so that
Lemma~\ref{lemma:ols} applies with $\sigma_0=\sqrt{L}$.
Applying Lemma~\ref{lemma:ols} with failure probability $\delta/3$ and the above
eigenvalue lower bound yields, simultaneously for all $i\in[I_k]$,
\[
|\bar a_i-a_i^\star|
\le \frac{2}{\Delta_{\min}^k}
\sqrt{\frac{2L\log(3I_k/\delta)}{N\,p_{\min}}}.
\]
A union bound over the three events gives failure probability at most
$\delta$. As the projection is in the convex set $C$, it is a contraction and we obtain:
\[
\max_{i}|\hat a_i-a_i^\star| \le \max_{i}|\bar a_i-a_i^\star|.
\]
Finally, since
$\|k-\hat k_N\|_\infty \le T_f \max_{i\in[I_k]}|\hat a_i-a_i^\star|$, we obtain
\[
\max_{\tau\ge 0}|k(\tau)-\hat k_N(\tau)|
\le
\frac{2T_f}{\Delta_{\min}^k}
\sqrt{\frac{2L\log(3I_k/\delta)}{N\,p_{\min}}}.
\]

\end{proof}

\newpage
\section{Proofs and remarks for Section~\ref{subsubsection:q} (concentration of $Q$)}
\label{appendix-Q-concentration}

\subsection*{Overall Problem and Result}

We model the probability probability $q$ of winning an auction as an increasing, piecewise linear function:
\begin{align*}
    \forall v \in [0,1]\,, \quad q(v) = q_{c^\star}(v) =\sum_{i = 1}^{I_q}c^\star_i \cdot (b_i \wedge v - b_{i-1} \wedge v)\,,
\end{align*}
With $c^\star_i > 0$, $\sum_{i} c^\star_i \cdot \Delta^q_i < 1 $ and $\Delta^q_i = b_i - b_{i-1}$. 

In second price auctions, $q$  is also the competition cumulative distribution, of density $f_{c^\star}(v) = \sum_{i = 1}^{I_q}c^\star_i \cdot \mathbb{I}[v \in (b_{i-1}, b_i]]\,.$

We add the following technical conditions:
\begin{itemize}
    \item $\forall i, \quad 0 < c_{\min} \le c^\star_i \le c_{\max} < \infty$.
    \item $\forall i, \quad 0 < \Delta^q_{\min} \le \Delta^q_i \le \Delta^q_{\max} < \infty$.
    \item We have for any bid $b \in [0, 1]$, $q(b) < 1 - \eta$, with $\eta > 0$.
\end{itemize}

In our setting, we suppose that we interact with $N$ sessions. After these $N$ sessions, we obtain $N_b$ number of auctions engaged, and collect data of the form:
\begin{align*}
    \mathcal{D}^q_{N_b} = \left\{ b_n, w_n, p_n \right\}_{n \in [N_b]}\,,
\end{align*}
with $b_n$ the bid placed, $w_n$ the binary win variable and $p_n$ the price paid, with the convention that $p_n = 0$ when $w_n = 0$.

Our objective is to estimate the slope vector $c^\star$ with all observed data, using MLE and finding:
\begin{align*}
    \hat{c} &= \arg\max_{c} \sum_{n \in [N_b]} \mathbb{I}[w_n = 1] \log f_c(p_n) + \mathbb{I}[w_n = 0] \log \left(1 - q_c(b_n) \right) \\
    &= \arg\max_{c} \sum_{n \in [N_b]} L_n(c)\,.
\end{align*}
This problem is concave and can be solved by a plethora of optimization algorithm \cite{conv_opt}. After $N$ sessions, we learn $\hat c$ and construct our $\hat{q}_N$ estimator with:
\begin{align*}
    \forall v \in [0,1]\,, \quad \hat{q}_N(v) =\sum_{i = 1}^{I_q}\hat{c}_i \cdot (b_i \wedge v - b_{i-1} \wedge v)\,,
\end{align*}

Let $\beta(\pi^\star) = \lim_{t\to+\infty}\pi^\star(t)$ denote the asymptotic bid value of the optimal policy, and $i_\star = I(\pi^\star)$ the smallest index $i$ such that $\beta(\pi^\star)\leq b_i$. With our $q$ extension theorem (Theorem~\ref{th:extension-theorem}), We want to accurately estimate $q$ on the bid range $[0, \beta(\pi^\star)]$, if we hope to identify the optimal policy. This is equivalent to controlling $|| \hat{c}_{:i_\star} - c^\star_{:i_\star} ||_2$ with $c_{:i}$ the vector of the first $i$ coordinates of $c$. In this section, we show that:
\begin{tcolorbox}
Let $\{\pi_n\}_{n \in [N]}$ be the bidders sequence, with each $\pi_n$ allowed to depend on the previously collected data $\mathcal{D}^k_{n}$.
\begin{condition}\label{cond:q}
 There exists $T_0 > 0$ s.t. for all $t \ge T_0\,,\, \pi_n(t) \ge \beta(\pi^\star)$ for all $n \in [N]$.
\end{condition}

Under Condition~\ref{cond:q}, for $N$ large enough, we have with high probability:
\begin{align*}
    \max_{v \in [0, \beta(\pi^\star)]}\,\lvert q(v) - \ \hat{q}_N(v)\rvert \le \mathcal{O}\left(\sqrt{1/N}\right)\,.
\end{align*}    
\end{tcolorbox}
To prove this result, we follow the recipe below and suppose that Condition~\ref{cond:q} holds in the rest of the proofs.

\subsection*{Lemma to explicit dependency on the number of auctions}
\begin{tcolorbox}
\begin{lemma}\label{lemma:N_b}

Let $N_b$ the r.v representing the number of auctions. 

For any $\delta \in (0,1)$, if $N \ge 32\log(2/\delta))$ then we have with probability $1 - \delta$:
\begin{align*}
    N_b &\ge N \frac{\mu}{2\gamma}\,.
\end{align*}
    
\end{lemma}
\end{tcolorbox}

\begin{proof}
Let $T_n\sim\mathrm{Exp}(\gamma)$ be the quit time of the session $n$ and,
conditional on $T_n$, let $B_n\mid T_n\sim\mathrm{Poisson}(\mu T_n)$ be the number
of arriving auctions. Define $N_b=\sum_{n=1}^N B_n$ and $S_T=\sum_{n=1}^N T_n$.
Conditionally on $(T_n)$, we have $N_b\sim\mathrm{Poisson}(\Lambda)$ with
$\Lambda=\mu S_T$.
By a Chernoff bound for $S_T$, with probability at least
$1-\delta/2$,
\[
S_T \ge \frac{N}{\gamma}\Big(1-\sqrt{\tfrac{2\log(2/\delta)}{N}}\Big),
\]
hence
\[
\Lambda \ge \frac{\mu N}{\gamma}\Big(1-\sqrt{\tfrac{2\log(2/\delta)}{N}}\Big).
\]
Conditionally on $\Lambda$, a Poisson lower-tail bound yields with probability at
least $1-\delta/2$,
\[
N_b \ge \Lambda-\sqrt{2\Lambda\log(2/\delta)}.
\]
By a union bound, both events hold with probability at least $1-\delta$, and on
this event
\[
N_b \ge \frac{\mu N}{\gamma}\Big(1-\sqrt{\tfrac{2\log(2/\delta)}{N}}
-\sqrt{\gamma\tfrac{2\log(2/\delta)}{\mu N}}\Big).
\]
If $\mu>2\gamma$, then $\sqrt{\gamma/\mu}\le 1/\sqrt{2}$ and the two square-root
terms are bounded by $2\sqrt{2\log(2/\delta)/N}$, giving
\[
N_b \ge \frac{\mu N}{\gamma}\Big(1-2\sqrt{\tfrac{2\log(2/\delta)}{N}}\Big).
\]
Finally, if $N>32\log(2/\delta)$, then $2\sqrt{2\log(2/\delta)/N}\le 1/2$, and thus
\[
N_b \ge \frac{\mu N}{2\gamma}.
\]
\end{proof}

\subsection*{Lemmas to prove learnability}

We denote by $\nabla_{:i}$ the gradient operator on the first $i$ coordinates. We define $\kappa_{i_\star} = \lambda_{\min}(H^{i_\star}(c^\star_{i_\star}))$ with $H^{i_\star}(c^\star_{i_\star}) = -\nabla^2_{:i_\star} \mathbb{E}[L(c_{:i_\star})]$ being the negative hessian of the population log-likelihood at its maximizer restricted to first $i_\star$ coordinates.
\begin{tcolorbox}
\begin{lemma}\label{lemma:mle}
The error between the MLE estimate and the true slope parameter can be upper bounded by:
\begin{align*}
    || \hat{c}_{:i_\star} - c^\star_{:i_\star} ||_2 \le \frac{||g^{:i_\star}_{N_b}(c^\star_{:i_\star})||_2}{\lambda_{\min}(\tilde{H}^{:i_\star}_{N_b})}
\end{align*}
with:
\begin{align*}
    g^{:i_\star}_{N_b}(c_{:i_\star}) &= \sum_{n \in [N_b]} \nabla_{:i_\star} L_n(c_{:i_\star})\,, \\
    \tilde{H}^{:i_\star}_{N_b} &= \int_{0}^1 \hat{H}^{:i_\star}_{N_b}(\hat{c}_{:i_\star} + s(c^\star_{:i_\star} - \hat{c}_{:i_\star}))ds\,, \text{ with } \hat{H}^{:i_\star}_{N_b}(c_{:i_\star}) = \sum_{n \in [N_b]} \nabla^2_{{:i_\star}} L_n(c_{:i_\star})\,, 
\end{align*}
\end{lemma}
\end{tcolorbox}

\begin{proof}
Let $\hat c_{:i_\star}$ denote the (restricted) MLE, i.e.
\[
g^{:i_\star}_{N_b}(\hat c_{:i_\star})
=\sum_{n=1}^{N_b}\nabla_{:i_\star} L_n(\hat c_{:i_\star})=0.
\]
Apply the vector Taylor expansion with integral remainder to the map
$g^{:i_\star}_{N_b}(\cdot)$ between $\hat c_{:i_\star}$ and $c^\star_{:i_\star}$:
\[
g^{:i_\star}_{N_b}(c^\star_{:i_\star})
=
g^{:i_\star}_{N_b}(\hat c_{:i_\star})
+
\left(\int_0^1 \hat H^{:i_\star}_{N_b}\!\big(\hat c_{:i_\star}
+s(c^\star_{:i_\star}-\hat c_{:i_\star})\big)\,ds\right)
(c^\star_{:i_\star}-\hat c_{:i_\star}),
\]
where $\hat H^{:i_\star}_{N_b}(c_{:i_\star})=\sum_{n=1}^{N_b}\nabla^2_{:i_\star} L_n(c_{:i_\star})$
is the Hessian of the sample log-likelihood restricted to the first $i_\star$ coordinates.
Using $g^{:i_\star}_{N_b}(\hat c_{:i_\star})=0$ and the definition of
\[
\tilde H^{:i_\star}_{N_b}
:=\int_0^1 \hat H^{:i_\star}_{N_b}\!\big(\hat c_{:i_\star}
+s(c^\star_{:i_\star}-\hat c_{:i_\star})\big)\,ds,
\]
we obtain the identity
\[
g^{:i_\star}_{N_b}(c^\star_{:i_\star})
=
\tilde H^{:i_\star}_{N_b}\,(c^\star_{:i_\star}-\hat c_{:i_\star}).
\]
Assuming $\tilde H^{:i_\star}_{N_b}$ is invertible, we can rewrite
\[
\hat c_{:i_\star}-c^\star_{:i_\star}
=
\big(\tilde H^{:i_\star}_{N_b}\big)^{-1}\, g^{:i_\star}_{N_b}(c^\star_{:i_\star}).
\]
Taking Euclidean norms and using $\|A^{-1}\|_{\mathrm{op}} = 1/\lambda_{\min}(A)$
for symmetric positive definite $A$, we conclude
\[
\|\hat c_{:i_\star}-c^\star_{:i_\star}\|_2
\le
\|\big(\tilde H^{:i_\star}_{N_b}\big)^{-1}\|_{\mathrm{op}}\,
\|g^{:i_\star}_{N_b}(c^\star_{:i_\star})\|_2
=
\frac{\|g^{:i_\star}_{N_b}(c^\star_{:i_\star})\|_2}
{\lambda_{\min}(\tilde H^{:i_\star}_{N_b})},
\]
which proves the claim.
\end{proof}

This result shows that if we bound the gradient and the minimum eigenvalue of the hessian, then we obtain a concentration on the slopes.


\begin{tcolorbox}
\begin{lemma}{Bounded Gradient and Hessian.} 

We have for any $n$:
\begin{align*}
    &||\nabla_{:i_\star} L_n(c^\star_{:i_\star})||_2 \le G_1 = \max\left\{\frac{1}{c_{\min}}, \sqrt{I_q} \frac{\Delta^q_{\max}}{\eta} \right\}  \\
    &||\nabla^2_{:i_\star} L_n(c^\star_{:i_\star})||_{\textsc{op}} \le G_2 =\max\left\{ \frac{1}{c_{\min}^2}, I_q \frac{(\Delta^q_{\max})^2}{\eta^2}  \right\}
\end{align*}
    
\end{lemma}
\end{tcolorbox}

\begin{proof}
    Calculus and exploit the conditions on $c^\star$, $\Delta^q$ and $q$.
\end{proof}

\subsection*{Gradient Result}
Now we give the Gradient Concentration with vector Hoeffding-Azuma (Lemma~\ref{lemma:hoeffding_azuma}).
\begin{tcolorbox}
\begin{lemma}{Gradient Concentration}\label{lemma:grad}
We have the following result with probability $1 - \delta$:
\begin{align*}
    ||g^{:i_\star}_{N_b}(c^\star_{:i_\star})||_2 \le G_1\sqrt{2 N_b \log 2/\delta}
\end{align*}
\end{lemma}
\end{tcolorbox}

\subsection*{Hessian Result}

\begin{tcolorbox}
\begin{lemma}[Hessian Concentration.]

Let $\delta \in (0, 1)$. We define: 
\begin{align*}
    &\Delta^q(\pi^\star) = \beta(\pi^\star) - b_{i_\star-1} > 0 \\
    &\kappa_\star = \frac{\Delta^q(\pi^\star)}{c_{\max}} \exp(-\mu T_0) > 0
\end{align*}
For $N_b > \frac{16 G_2 \log(I_q/\delta)}{\kappa_\star}$, and under Condition~\ref{cond:q}, we have with probability at least $1 - \delta$:
\begin{align*}
    \lambda_{\min}(\hat{H}^{:i_\star}_{N_b}(c^\star_{:i_\star})) \ge  \frac{N_b}{2}\kappa_{\star} \,.
\end{align*} 
\end{lemma}
\end{tcolorbox}

\begin{proof}
   We invoke Matrix Chernoff Bound (Lemma~\ref{lemma:matrix_chernoff}) and control the minimum eigen value of the sum of conditional expectations $M = \sum_{n = 1}^{N_b} \mathbb{E}[X_n \mid \mathcal{F}_{n-1}]$ with $X_n = \hat{H}^{:i_\star}_{n}(c^\star_{:i_\star})$. 
   
   Conditioning on the filtration fixes the bidding strategies used. The expectation in our case is on the time $\tau$, the competition price $p$ and the win $w$. For these variables, the log-likelihood is written as:
\begin{align*}
    L_n(c) = \mathbb{I}[w_n = 1] \log f_c(p_n)  + \mathbb{I}[w_n = 0] \log (1 - q_c(\pi_{n-1}(\tau_n)))\,.
\end{align*}
Both win and loss terms are concave in $c$, as they are log-transforms of affine functions in $c$, this means that we can focus on the win term alone as we have:
\begin{align*}
    \lambda_{\min}\left(\mathbb{E}_{n-1}[X_n \mid \mathcal{F}_{n-1}]\right) &\ge \mathbb{E}[\lambda_{\min}(-\nabla^2_{:i_\star} (\mathbb{I}[w = 1] \log f_{c_{:i_\star}^\star}(p)))] \\
    &= \lambda_{\min}(-\nabla^2_{:i_\star} \mathbb{E}_{n-1}[\mathbb{I}[w_n = 1] \log f_{c_{:i_\star}}(p_{n})](c^\star_{:i_\star}))\,.
\end{align*}

We have $H^{:i_\star}_{\text{win}} = -\nabla^2_{:i_\star} \mathbb{E}_{n-1}[\mathbb{I}[w_n = 1] \log f_{c_{:i_\star}}(p_n)](c^\star_{:i_\star})$ a diagonal matrix, with:
$$H^{:i_\star}_{\text{win}, ii} = \frac{P_{n-1}(W_n=1, P_n \in (b_{i-1}, b_i])}{(c_i^\star)^2}\,,$$
and $0$ otherwise, meaning that:
\begin{align*}
    \lambda_{\min}\left(\mathbb{E}_{n-1}[X_n \mid \mathcal{F}_{n-1}]\right) \ge \min_{1 \le i \le i_\star} \frac{P_{n-1}(W_n=1, P_n \in (b_{i-1}, b_i])}{(c_i^\star)^2} \,.
\end{align*}
Looking at $P_{n-1}(W_n=1, P_n \in (b_{i-1}, b_i])$, we have that:
\begin{align*}
    P_{n-1}(W_{n}=1, P_{n} \in (b_{i-1}, b_i]) &= P(\pi_{n-1}(\tau) > P_n, P_n \in (b_{i-1}, b_i]) \\
    &= \int_{b_{i-1}}^{b_i}P(\pi_{n-1}(\tau) > p) f(p) dp \\
    &= \int_{b_{i-1}}^{b_i}P(\pi_{n-1}(\tau) > p) c^\star_i dp \,.
\end{align*}
This gives
\begin{align*}
    \lambda_{\min}\left(\mathbb{E}_{n-1}[X_n \mid \mathcal{F}_{n-1}]\right) &\ge \frac{\int_{b_{i_\star-1}}^{b_{i_\star}}P(\pi_{n-1}(\tau) > p) dp}{c_{\max}} \\
    &\ge \frac{\int_{b_{i_\star-1}}^{\beta(\pi^\star)}P(\pi_{n-1}(\tau) > p) dp}{c_{\max}} \\
    &\ge P(\pi_{n-1}(\tau) > \beta(\pi^\star))\frac{\beta(\pi^\star) - b_{i_\star - 1}}{c_{\max}} \\
    &\ge \frac{\Delta^q(\pi^\star)}{c_{\max}} \exp(-\mu T_0)\,. \quad (\text{Condition~\ref{cond:q}})
\end{align*}
We then obtain $\lambda_{\min}(M) \ge N_b \kappa_\star$, and means that choosing $N_b$ such that  $N_b > \frac{16 G_2 \log(I_q/\delta)}{\kappa_{\star}}$, yields the following with probability $1-\delta$:
\begin{align*}
    \lambda_{\min}(\hat{H}^{:i_\star}_{N_b}(c^\star_{:i_\star})) \ge  \frac{N_b}{2}\kappa_{\star} \,.
\end{align*} 
which ends the proof.
\end{proof}

Finally, to control the $\ell_2$ norm of our estimate in Lemma~\ref{lemma:mle}, we want to control the minimum eigenvalue of a different PSD matrix $\tilde{H}_{N_b}$. We prove the Lipchitzness of the empirical hessian first

\begin{tcolorbox}
    \begin{lemma}[Hessian Lipchitzness]{\label{lemma:lipchitz_hessian}}

Suppose that $\forall i,\, \hat{c}_i \ge c_{\min}$. We have that:
\begin{align*}
||\tilde{H}^{:i_\star}_{N_b} - \hat{H}^{:i_\star}_{N_b}(c^\star_{:i_\star})||_{\text{op}} &\le \sup_{s \in [0, 1]} || \hat{H}_{N_b}(c^\star_{:i_\star} + s(\hat{c}_{:i_\star} - c^\star_{:i_\star})) - \hat{H}_{N_b}(c^\star_{:i_\star}) ||_{\textsc{op}} \\
&\le G_3 N_b || \hat{c}_{:i_\star} - c^\star_{:i_\star} ||_2
\end{align*}
with $G_3 = \max\left\{\frac{2}{c^3_{\min}}, 2\frac{I_q^{3/2} \Delta^3_{\max}}{\eta^3}  \right\}.$
\end{lemma}
\end{tcolorbox}

\begin{proof}
    We apply the mean value theorem for the Hessian and bound its third derivative with $G_3$.
\end{proof}

\subsection*{The Final Concentration Result for $q$}

Now to finish our cooking recipe, we combine Weyl Lemma (Lemma~\ref{lemma:weyl}), Hessian Lipchitzness (Lemma~\ref{lemma:lipchitz_hessian}), Dependency on $N$ (Lemma~\ref{lemma:N_b}), MLE Bound (Lemma~\ref{lemma:mle}) to obtain the following, core proposition:
\begin{tcolorbox}
\begin{proposition}\label{prop:concentration_q}
Let $\delta \in (0, 1)$. We define
$$N^q_0 = \max \left\{32\log(6/\delta), 4 \frac{\gamma}{\mu} \log(6/\delta) \left(\frac{4 G_1 G_3}{\kappa_\star^2} \right)^2, \frac{32 \gamma \log(3I_q/\delta)}{\mu} \right\}$$
If $N \ge N^q_0$, then we have with probability at least $1 - \delta$:
\begin{align*}
    || \hat{c}_{:i_\star} - c^\star_{:i_\star} ||_2 \le \frac{4 G_1}{\kappa_{\star}} \sqrt{\frac{\gamma \log 6/\delta}{\mu N}}\,,
\end{align*}
which also gives the following:
\begin{align*}
    \forall v \in [0, \beta(\pi^\star)],\, |q(v) - \hat{q}(v)| \le \frac{4 G_1}{\kappa_{\star}} \sqrt{\frac{\gamma \log 6/\delta}{\mu N}}\,.
\end{align*}

\end{proposition}
\end{tcolorbox}
\begin{proof}
Let $0 < \delta < 1$. If $N_b \ge \frac{16 G_2 \log(3 I_q/\delta)}{k_\star}$, we have with probability $1 - \delta/3$:
\begin{align*}
    \lambda_{\min}(\hat{H}^{:i_\star}_{N_b}(c^\star_{:i_\star})) \ge  \frac{N_b}{2}\kappa_{\star} \,.
\end{align*}
Using Weyl's lemma (Lemma~\ref{lemma:weyl}) combined with Lemma~\ref{lemma:mle} and Hessian Lipchitzness (Lemma~\ref{lemma:lipchitz_hessian}), we have:
\begin{align*}
    || \hat{c}_{:i_\star} - c^\star_{:i_\star} ||_2 \le \frac{||g^{:i_\star}_{N_b}(c^\star_{:i_\star})||_2}{N_b (\kappa_\star - G_3 || \hat{c}_{:i_\star} - c^\star_{:i_\star} ||_2)}\,.
\end{align*}
Solving for $|| \hat{c}_{:i_\star} - c^\star_{:i_\star} ||_2$, we get:
\begin{align*}
    || \hat{c}_{:i_\star} - c^\star_{:i_\star} ||_2 \le \frac{2||g^{:i_\star}_{N_b}(c^\star_{:i_\star})||_2}{N_b\kappa_\star},
\end{align*}
once $|| \hat{c}_{:i_\star} - c^\star_{:i_\star} ||_2 \le \kappa_\star/2G_3$.
Using Lemma~\ref{lemma:grad}, we have with probability $1 - \delta/3$:
\begin{align*}
    ||g^{:i_\star}_{N_b}(c^\star_{:i_\star})||_2 \le G_1\sqrt{2 N_b \log 6/\delta}
\end{align*}
giving with probability $1 - 2\delta/3$:
\begin{align*}
    || \hat{c}_{:i_\star} - c^\star_{:i_\star} ||_2 \le \frac{2G_1}{\kappa_\star}\sqrt{\frac{2 \log 6/\delta}{N_b}},
\end{align*}
as long as $N_b \ge 2 \log(6/\delta) \left(\frac{4 G_1 G_3}{\kappa_\star^2} \right)^2$. Finally, we link $N_b$ to the number of episodes $N$, having with probability $1 - \delta/3$:
\begin{align*}
    N_b \ge \frac{\mu}{2\gamma}N 
\end{align*}
once $N \ge 32\log(6/\delta)$. A union bound gets us once $N \ge \max \left\{32\log(6/\delta), 4 \frac{\gamma}{\mu} \log(6/\delta) \left(\frac{4 G_1 G_3}{\kappa_\star^2} \right)^2, \frac{32 \gamma \log(3I_q/\delta)}{\mu} \right\}$:
\begin{align*}
    || \hat{c}_{:i_\star} - c^\star_{:i_\star} ||_2 \le \frac{4G_1}{\kappa_\star}\sqrt{\frac{\gamma\log 6/\delta}{\mu N}},
\end{align*}
proving our result.
\end{proof}



    


\newpage

\newpage

\newpage

\newpage
\section{Analysis of the algorithms}
    \label{appendix-algo}

We have seen in the concentration results that to ensure concentration of both our estimators of $q$ and $k$ we need to respect Conditions \ref{cond:k} and \ref{cond:q} ensuring that our bidders $\pi_n$ for each episode bid enough in $t_1$ (to learn $k$) and bid high enough (to learn $q$). Under these conditions, we get for $N$ large enough the high probability bounds:
\begin{align*}
    \lVert k - \hat{k}_N \rVert_\infty &\le C_k/\sqrt{N} \\
    \max_{v \in [0, \beta(\pi^\star)]}\lvert q(v) - \hat{q}_N(v) \rvert &\le C_q/\sqrt{N}\,.
\end{align*}
We take $C = \max\{C_k, C_q\}$ and define $\epsilon_N = C/\sqrt{N}$.

We have by definition $\pi^\star = S(k, q)$, and we define $\pi_N = S(\hat{k}_N, \hat{q}_N)$. We start by giving the following lemma:
\begin{tcolorbox}
\begin{lemma}\label{lemma:reg_quad}
For any $\pi$, we have:
\begin{align*}
    ||V^\star_{.}-V^{\pi}_{.}||_{\infty}\leq \frac{\mu}{\gamma}\lVert D_Q(\pi^\star, \pi) \rVert_\infty \le \frac{\mu}{\gamma}\frac{c_{\max}}{2}\lVert \pi^\star - \pi \rVert^2_\infty\,.
\end{align*}
\end{lemma}
\end{tcolorbox}
\begin{proof}
The first result $||V^\star_{.}-V^{\pi}_{.}||_{\infty}\leq \frac{\mu}{\gamma}\lVert D_Q(\pi^\star, \pi) \rVert_\infty$ is obtained from Lemma~\ref{lemma:value_bregman}. As $Q$ is twice differentiable almost everywhere with $Q''= q'$, we can use mean value theorem, with $|q'(v)| \le c_{\max}$ and end the proof.
\end{proof}
We now prove the following result:
\begin{tcolorbox}
\begin{lemma}\label{lemma:eps_pol}
For $N$ large enough (small $\epsilon$), we have:
\begin{align*}
    \lVert\pi^\star - \pi_N \rVert_\infty \le \mathcal{O}(1/\sqrt{N})\,.
\end{align*}
\end{lemma}
\end{tcolorbox}
\begin{proof}
Let $\lVert\pi^\star - \pi_N \rVert_\infty = \Delta_N$. For any $\tau \ge 0$, and by definition of $S(k,q)$, we have:
\begin{align*}
    \pi^\star(\tau) - \pi_N(\tau) &= (k(\tau) - \hat{k}_N(\tau)) + (V^\star(0) - \hat{V}_N(0)) + (V^\star(\tau) - \hat{V}_N(\tau)) \\
    &\le \lVert k - \hat{k}_N\rVert_\infty + 2 \lVert V^\star - \hat{V}_N \rVert_\infty \\
    &\le \lVert k - \hat{k}_N\rVert_\infty + 2 (\lVert V^\star - V^{\pi_N} \rVert_\infty + \lVert V^{\pi_N} - \hat{V}_N \rVert_\infty)\,.
\end{align*}
Combining Lemma~\ref{lemma:reg_quad} and Lemma~\ref{lemma:dynamicsensitivity}, we get:
\begin{align*}
    \Delta_N &\le \epsilon_N + 2 (\frac{\mu}{\gamma} \lVert D_Q(\pi^\star, \pi_N)\rVert_\infty + 4\frac{\mu^2}{\gamma^2}\epsilon_N) \\
    &\le \epsilon_N + 2 (\frac{\mu}{\gamma} \frac{c_{\max}}{2} \Delta^2_N + 4\frac{\mu^2}{\gamma^2}\epsilon_N) \\
    &\le (1 + 8\frac{\mu^2}{\gamma^2})\epsilon_N + \frac{\mu}{\gamma} c_{\max} \Delta^2_N 
\end{align*}
if $4\frac{\mu}{\gamma} c_{\max} (1 + 8\frac{\mu^2}{\gamma^2})\epsilon_N < 1 $, meaning that $N$ is large enough, we get:
$$\Delta_N \le 2 (1 + 8\frac{\mu^2}{\gamma^2})\epsilon_N\,,$$
ending the proof.
\end{proof}

Now we have all ingredients to analyze our algorithms. We work in the fixed horizon setting.

\subsection{Two-Phase Algorithm: Learn then Rollout}

\begin{tcolorbox}
\begin{theorem}
\label{th:appendix:2phase}
Consider the two phase learning algorithm:
\begin{enumerate}
    \item \textbf{Exploration Phase ($N_1$ episodes):} Use a fixed exploratory policy $\pi_{\text{x}}$ that bids $k(\infty)$ (the maximal value) to collect data for estimating \( k \) and \( q \). Compute $\hat{k} = K(\mathcal{D}^k_{N_1})$ and $\hat{q}=Q(\mathcal{D}^q_{N_1})$.
    \item \textbf{Exploitation Phase ($N_2 = N - N_1$ episodes):} Use policy $\pi = \mathcal{S}(\hat{k}, \hat{q})$.
\end{enumerate}    
The algorithm explores with $N_1$ episodes and exploits the rest of the $N_2$ sessions.

Let $\delta \in (0, 1)$. There exists $C_0$ such that once $N \geq C_0$, setting $N_1 = \mathcal{O}(\sqrt{N})$ yields with probability at least $1-\delta$:
\begin{align}
    \text{Regret}(N) \leq \mathcal{O}(\sqrt{N}).
\end{align}
\end{theorem}
\end{tcolorbox}

\begin{proof}
Let us set $N_1 = \sqrt{N}$ and $C_0 = \max\{N^k_0, N^q_0\}^2$ with $N^k_0, N^q_0$ respectively defined in Proposition~\ref{prop:concentration_k} and Proposition~\ref{prop:concentration_q}. 

For the first exploration phase, we have $\pi_{\texttt{x}} = k(\infty)$ respecting both conditions \ref{cond:k} and \ref{cond:q}. 

If $N \ge C_0$, then $N_1 \ge \max\{N^k_0, N^q_0\}$ and we can use both Proposition~\ref{prop:concentration_k} and Proposition~\ref{prop:concentration_q}. 

These ensure us that with probability at least $1-\delta$:
\begin{align*}
   &\lVert k - \hat{k} \rVert_\infty \le \mathcal{O}(1/\sqrt{N_1})\,,
   &\lVert q - \hat{q} \rVert_\infty \le \mathcal{O}(1/\sqrt{N_1})
\end{align*}
giving also $\lVert \pi^\star - \hat{\pi} \rVert_{\infty} \le \mathcal{O}(1/\sqrt{N_1})$. This gives us a $\lVert V^\star - V^{\hat{\pi}} \rVert_{\infty} \le \mathcal{O}(1/N_1)$ (Lemma~\ref{lemma:reg_quad}) and a total regret of:
\begin{align}
    \text{Regret}(N) &\leq \underbrace{\mathcal{O}(N_1)}_{\text{exploration}} + \underbrace{\mathcal{O}(N/N_1)}_{\text{exploitation}} \\
    &\leq \mathcal{O}(\sqrt{N})\,.
\end{align}
\end{proof}

\subsection{Three-Phase Algorithm: Learn $k$ then $q$ then Rollout}
\label{appendix:section:3phase}

\begin{tcolorbox}
\begin{theorem}
\label{proof:three-phase}
Consider the three phase learning algorithm:
\begin{enumerate}
    \item \textbf{Exploration Phase for $k$ ($N_1$ episodes):} Use a fixed exploratory policy $\pi_{\texttt{x}}$ that bids $b_0 > 0$ to collect data for estimating \( k \). Compute $\hat{k} = K(\mathcal{D}^k_{N_1})$.
    \item \textbf{Exploration Phase for $q$ ($N_2$ episodes):} Use the learned $\pi_k =  \hat{k}$ to collect data for estimating \( q \). Compute $\hat{q} = Q(\mathcal{D}^q_{N_2})$.
    \item \textbf{Exploitation Phase ($N_3 = N - N_1 - N_2$ episodes):} Use policy $\pi = \mathcal{S}(\hat{k}, \hat{q})$.
\end{enumerate}    
The algorithm learns $\hat{k}$, uses $\hat{k}$ to learn $\hat q$ then exploits with $\pi = \mathcal{S}(\hat{k}, \hat{q})$ the rest of the $N$ sessions.

Let $\delta \in (0, 1)$. There exists $C_0$ such that once $N \geq C_0$, setting $N_1 = N_2 =  \mathcal{O}(\sqrt{N})$ yields with probability at least $1-\delta$:
\begin{align*}
    \text{Regret}(N) \leq \mathcal{O}(\sqrt{N}).
\end{align*}
\end{theorem}
\end{tcolorbox}
\begin{proof}
Let us set $N_1 = N_2 =\sqrt{N}$ and $C_0 = \max\{N^k_0, N^q_0\}^2$ with $N^k_0, N^q_0$ respectively defined in Proposition~\ref{prop:concentration_k} and Proposition~\ref{prop:concentration_q}. In the first phase, we bid $b_0 > 0$ for $N_1$ episodes, as $N_1 \ge N^k_0$, we have with probability $1-\delta$:
\begin{align*}
   \lVert k - \hat{k} \rVert_\infty \le C_k/\sqrt{N_1}\,,
\end{align*}
with $C_k$ a constant. As $\mu > 2\gamma$, then $k(\tau) > \pi^\star(\tau)$ for all $\tau$. We set $\Delta_k = \min_\tau (k(\tau) - \pi^\star(\tau))/2$. If $N_1 \ge (C_k/\Delta_k)^2$, we obtain with probability $1 - \delta$:
\begin{align*}
    \hat{k}(\tau) \ge k(\tau) - C_k/\sqrt{N_1} \ge k(\tau) - \Delta_k = \frac{k(\tau) + \pi^\star(\tau)}{2} > \pi^\star(\tau)\,.
\end{align*}
This means that with probability $1-\delta$, $\hat{k}$ upper bounds $\pi^\star$ and that it verifies condition~\ref{cond:q}. Learning $\hat{q}$ on $N_2$ sessions collected with $\hat{k}$ ensures that with probability $1-\delta$:
\begin{align*}
    \lVert q - \hat{q} \rVert_\infty \le \mathcal{O}(1/\sqrt{N_1})\,.
\end{align*}
We thus obtain $\lVert \pi^\star - \hat{\pi} \rVert_{\infty} \le \mathcal{O}(1/\sqrt{N_1})$,  $\lVert V^\star - V^{\hat{\pi}} \rVert_{\infty} \le \mathcal{O}(1/N_1)$ (Lemma~\ref{lemma:reg_quad}) and a total regret of:
\begin{align}
    \text{Regret}(N) &\leq \underbrace{\mathcal{O}(N_1)}_{\text{exploration}} + \underbrace{\mathcal{O}(N_1)}_{\text{exploration with $\hat{k}$}} + \underbrace{\mathcal{O}(N/(N_1 + N_2))}_{\text{exploitation}} \\
    &\leq \mathcal{O}(\sqrt{N})\,.
\end{align}

\end{proof}

\subsection{Confidence Bounds Algorithm}

\begin{tcolorbox}
\begin{lemma}[Time-uniform finite-LIL confidence envelopes]
\label{lem:lil_confidence_envelopes}
Assume that the estimators \(K(\mathcal D^k_n)\) and \(Q(\mathcal D^q_n)\)
satisfy the martingale/sub-Gaussian regularity conditions used in
Propositions~\ref{prop:concentration_k} and~\ref{prop:concentration_q}.
Then there exist problem-dependent constants
\(C_k,C_q,c_{\mathrm{LIL}}>0\) and \(C_0\in \mathbb N\) such that, for every
\(\delta\in(0,1)\), with probability at least \(1-\delta\), simultaneously
for all \(n\ge C_0\),

\[
    \left\|K(\mathcal D^k_n)-k\right\|_{\infty}
    \le
    C_k
    \sqrt{
        \frac{
            1+\log\log(n)+\log(c_{\mathrm{LIL}}/\delta)
        }{n}
    },
\]

and

\[
    \left\|Q(\mathcal D^q_n)-q\right\|_{\infty}
    \le
    C_q
    \sqrt{
        \frac{
            1+\log\log(n)+\log(c_{\mathrm{LIL}}/\delta)
        }{n}
    }.
\]

Equivalently, the confidence envelopes hold uniformly over time without
taking a union bound over \(n=1,\dots,N\).
\end{lemma}
\end{tcolorbox}

\begin{proof}
This is a direct application of the finite-LIL confidence sequence of
\cite[Theorem~1 and Eq.~(3.4)]{anytime_bound}.

Under the martingale/sub-Gaussian regularity assumptions of
Propositions~\ref{prop:concentration_k} and~\ref{prop:concentration_q},
the centered estimation errors associated with
\(K(\mathcal D^k_n)-k\) and \(Q(\mathcal D^q_n)-q\) satisfy
time-uniform sub-Gaussian martingale concentration with variance proxy
of order \(n\). Applying the two-sided polynomial stitched confidence
sequence of \citet{anytime_bound}, and allocating failure
probability \(\delta/2\) to the \(k\)-estimator and \(\delta/2\) to the
\(q\)-estimator, yields constants
\(C_k,C_q,c_{\mathrm{LIL}}>0\) such that, with probability at least
\(1-\delta\), simultaneously for all \(n\ge C_0\),

\[
    \left\|K(\mathcal D^k_n)-k\right\|_{\infty}
    \le
    C_k
    \sqrt{
        \frac{
            1+\log\log(n)+\log(c_{\mathrm{LIL}}/\delta)
        }{n}
    },
\]

and

\[
    \left\|Q(\mathcal D^q_n)-q\right\|_{\infty}
    \le
    C_q
    \sqrt{
        \frac{
            1+\log\log(n)+\log(c_{\mathrm{LIL}}/\delta)
        }{n}
    }.
\]

The constants absorb the numerical constants in the stitched boundary,
the two-sided symmetrization, the split of the failure probability
between the two estimators, and the constants appearing in
Propositions~\ref{prop:concentration_k} and
\ref{prop:concentration_q}. Since the confidence sequence is already
uniform over all times, no union bound over \(n=1,\dots,N\) is needed.
\end{proof}

\begin{tcolorbox}
\begin{theorem}[Confidence Bounds Algorithm]
\label{proof-ucb}
Choose $\lambda_k, \lambda_q >0$. For each episode $n \in [1, N]$ do:
    \begin{itemize}
        \item $\hat{k}^{\texttt{UCB}}_n = \min(K(\mathcal{D}^k_n) + \lambda_k \sqrt{\log \log n/n}\,,\, k(\infty))$
        \item $ \hat{q}^{\texttt{LCB}}_n = \max(Q(\mathcal{D}^q_n) - \lambda_q \sqrt{\log \log n/n} \,,\, 0)$
        \item $\pi_n = \mathcal{S}(\hat{k}^{\texttt{UCB}}_n, \hat{q}^{\texttt{LCB}}_n)$
        \item $\mathcal{D}^k_{n+1} = \mathcal{D}^k_{n} \cup \mathcal{D}^k_{\pi_n}$ and $\mathcal{D}^q_{n+1} = \mathcal{D}^q_{n} \cup \mathcal{D}^q_{\pi_n}$ with $\mathcal{D}^k_{\pi_n}$ and $\mathcal{D}^q_{\pi_n}$ the datasets generated at episode $n$.
    \end{itemize}   
The algorithm constructs an envelope around $q$ and $k$ and derives a policy that bids in the worst case environment.

Let $\delta \in (0, 1)$. There exists $\lambda_q$ and $\lambda_k$ (depending on problem constants) that makes the algorithm yield with probability at least $1-\delta$:
\begin{align}
    \text{Regret}(N) \leq  \tilde{\mathcal{O}}(\log(N)).
\end{align}
\end{theorem}
\end{tcolorbox}
\begin{proof}
Without loss of generality, let us suppose that $k(\infty) = 1$. Let $N^k_0$ and $N^q_0$ respectively be the number of sessions that kickstart the concentration of $k$ and $q$ in Proposition~\ref{prop:concentration_k} and Proposition~\ref{prop:concentration_q}. We set $C_0 = \max\{ N_0^k, N_0^q\}$. We choose $\lambda_k > C_0$ and $\lambda_q > C_0$ such that for all episodes before $C_0$, we get $\hat{k}^{\texttt{UCB}}_n = k(\infty)$ and all episodes before $C_0$ we get $\hat{q}^{\texttt{LCB}}_n = 0$, ensuring coverage as $\mathcal{S}(\hat{k}^{\texttt{UCB}}_n, \hat{q}^{\texttt{LCB}}_n) = k(\infty)$ in this scenario. Starting from $n \ge C_0$, we get the concentration bounds on $k_n$ and $q_n$, with the following holding with high probability for all $n$:
\begin{align*}
   &k_n(\tau) \le \hat{k}(\tau) + C_k \sqrt{\log \log n/n} \\
   &q_n(v) \ge \hat{q}(v) - C_q\sqrt{\log \log n/n}\,,
\end{align*}
Setting $\lambda_k = \max(C_0, C_k)$, and $\lambda_q = \max(C_0, C_q)$, we get using the monotony of $\mathcal{S}$ (Theorem~\ref{th:monotony-theorem}) that $\pi_n = \mathcal{S}(\hat{k}^{\texttt{UCB}}_n, \hat{q}^{\texttt{LCB}}_n)$ has a maximal bid that covers the maximal optimal bid $\beta(\pi^\star)$ with high probability starting from $\bar{T}$, and ensuring that the bound holds for $n+1$. This means that the argument follows by induction. We thus obtain $\lVert \pi^\star - \pi_n \rVert_{\infty} \le \mathcal{O}(\sqrt{\log \log n/n})$,  $\lVert V^\star - V^{\pi_n} \rVert_{\infty} \le \mathcal{O}(\log \log n/n)$ (Lemma~\ref{lemma:reg_quad}) and a total regret of:
\begin{align}
    \text{Regret}(N) &\leq \sum_{n}\mathcal{O}(\log \log n/n) \le \tilde{\mathcal{O}}(\log(N))\,,
\end{align}
which gives us the result.
\end{proof}

\section{General Primitives Results}

\begin{tcolorbox}
\begin{lemma}[Piecewise linear approximation]
\label{lem:piecewise_linear_approximation_app}
Let \(f:[a,b]\to\mathbb R\) be continuously differentiable and assume that
\(f'\) is \(L\)-Lipschitz. Let \(f_{PL,m}\) be the piecewise linear
interpolant of \(f\) on a uniform grid of \(m\) intervals on \([a,b]\).
Then
\[
    \|f-f_{PL,m}\|_\infty
    \le
    \frac{L(b-a)^2}{8m^2}.
\]
\end{lemma}
\end{tcolorbox}

\begin{proof}
Let \(h=(b-a)/m\) be the mesh size. On each interval
\([x_i,x_{i+1}]\), the standard interpolation error bound for functions
with Lipschitz derivative gives
\[
    |f(x)-f_{PL,m}(x)|
    \le
    \frac{Lh^2}{8},
    \qquad x\in[x_i,x_{i+1}].
\]
Taking the supremum over all grid intervals yields
\[
    \|f-f_{PL,m}\|_\infty
    \le
    \frac{Lh^2}{8}
    =
    \frac{L(b-a)^2}{8m^2}.
\]
\end{proof}

\begin{tcolorbox}
\begin{theorem}[Regret for general smooth primitives]
\label{thrm:regret_general}
Suppose that \(k\) and \(q\) are continuously differentiable and have
Lipschitz derivatives on their respective domains. Let \(N\) be the horizon,
and use the family of piecewise linear functions for \(k\) and \(q\) on uniform
grids with $m \asymp N^{1/6}$
intervals. 

Then the Confidence Bounds Algorithm satisfies, with high
probability, 

$$\operatorname{Regret}(N)
    \le
    \widetilde{\mathcal O}\!\left(N^{1/3}\right).$$
\end{theorem}
\end{tcolorbox}

\begin{proof}
Let \(k_{PL,m}\) and \(q_{PL,m}\) denote the piecewise linear interpolants of
\(k\) and \(q\) on uniform grids with \(m\) intervals. Since \(k\) and \(q\)
have Lipschitz derivatives, the standard interpolation bound gives
\[
    \|k-k_{PL,m}\|_\infty
    \le \mathcal O(m^{-2}), \quad 
    \|q-q_{PL,m}\|_\infty \le \mathcal O(m^{-2})\,.
\]
Moreover, applying the confidence bounds established for piecewise linear
primitives on a grid of size \(m\), we have, with high probability,
\begin{align*}
    \|k_{PL,m}-\hat k_n\|_\infty
    & \le
    \widetilde{\mathcal O}\!\left(\frac{1}{\Delta^k_{\min}}\sqrt{\frac1n}\right)=
    \widetilde{\mathcal O}\!\left(m\sqrt{\frac1n}\right) \\
    \|q_{PL,m}-\hat q_n\|_\infty
    & \le
    \widetilde{\mathcal O}\!\left(\frac{1}{\Delta^q_{\min}}\sqrt{\frac1n}\right)=
    \widetilde{\mathcal O}\!\left(m\sqrt{\frac1n}\right)
\end{align*}
Therefore, by the triangle inequality,

\begin{align*}
    \|k-\hat k_n\|_\infty
    & \le 
    \widetilde{\mathcal O}\!\left(m\sqrt{\frac1n} + \frac{1}{m^2}\right) \\
    \|q-\hat q_n\|_\infty
    &\le
    \widetilde{\mathcal O}\!\left(m\sqrt{\frac1n} + \frac{1}{m^2
    }\right)
\end{align*}

Fixing \(m\asymp N^{1/6}\), this becomes, by the quadratic regret bound of Proposition~\ref{prop:epsilon_opt}, the
per-episode regret satisfies
\[
    \operatorname{Reg}_n
    \le
    \widetilde{\mathcal O}\!\left(
        \left(
            N^{1/6}n^{-1/2}
            +
            N^{-1/3}
        \right)^2
    \right).
\]
Hence
\[
\begin{aligned}
    \operatorname{Regret}(N)
    &\le
    \sum_{n=1}^N
    \widetilde{\mathcal O}\!\left(
        \left(
            N^{1/6}n^{-1/2}
            +
            N^{-1/3}
        \right)^2
    \right) \\
    &\le
    \widetilde{\mathcal O}\!\left(
        N^{1/3}\sum_{n=1}^N\frac1n
        +
        N^{-1/6}\sum_{n=1}^N\frac1{\sqrt n}
        +
        N^{-2/3}\sum_{n=1}^N 1
    \right) \\
    &=
    \widetilde{\mathcal O}\!\left(N^{1/3}\right),
\end{aligned}
\]
where the logarithmic factor from \(\sum_{n=1}^N 1/n\) is absorbed into
\(\widetilde{\mathcal O}\). This proves the claim.
\end{proof}

\newpage
\section{Analysis of the lower bound}
    \label{appendix-lowerbound}

In what follows, we derive a logarithmic lower bound. We begin by presenting qualitative arguments and a proof sketch to build intuition for the result. However, we ultimately construct our formal proof by leveraging specialized results from the existing literature~\cite{weed2016online}, which provides a more direct path to the conclusion.

The  ingredients for the construction of the lower bound are Lemma~\ref{lemma:bregman-lower}, Lemma~\ref{lemma:b_dynamics},
and the  Bretagnolle-Huber inequality  mentioned below.

\begin{tcolorbox}
\begin{lemma}
\label{lemma:bregman-lower}
Let $\pi^\star$ be the optimal policy, for any  policy $\pi$,
    \begin{align*}
        V^\star(0) - V^\pi(0) \geq \frac{\mu  e^{-(\gamma+\mu) \bar{T}} }{\gamma +\mu}. D_Q(\beta(\pi^\star),\beta(\pi)),
    \end{align*}
    where $D_Q$ is the Bregman divergence of $Q:b\to \int_0^b q(b) \dd b $.
\end{lemma}
\end{tcolorbox}

In this proof, we use the subscript notation for function to improve readability.
\begin{proof}
Applying the dynamic programming principle (see \cite{heymann2023repeated}), 
$V^\star_0 - V^\pi_0$ is equal to 
    \begin{align*}
        \int_0^{+\infty} e^{-(\gamma+\mu) t}  \mu \Bigg[
        V^\star_{\tau+t}+U\left(k_{\tau+t}+V^\star_0-V^\star_{\tau+t}, \pi^\star_{\tau+t}\right)  -   
        V^\pi_{\tau+t}-U\left(k_{\tau+t}+V^\pi_{0}-V^\pi_{\tau+t}, \pi_{\tau+t}\right) \Bigg]\dd t. 
        \end{align*}
        Since the integrand is greater than 0, we can lower bound this integral by integrating only on $t$ greater than $\bar{T}$, which yields
        \begin{align*}
         \int_{\bar{T}}^{+\infty} e^{-(\gamma+\mu) t}  \mu \Bigg(
        V^\star_{\bar{T}}+U\left(k_{\bar{T}}+V^\star_0-V^\star_{\bar{T}}, \pi^\star_{\bar{T}}\right) - 
        V^\pi_{\bar{T}}-U\left(k_{\bar{T}}+V^\pi_0-V^\pi_{\bar{T}}, \pi_{\bar{T}}\right)\Bigg) \dd t \\
      \geq  \frac{\mu  e^{-(\gamma+\mu) \bar{T}} }{\gamma +\mu} \Bigg(
      U\left(k_{\bar{T}}+V^\star_0-V^\star_{\bar{T}}, \pi^\star_{\bar{T}}\right) - 
        U\left(k_{\bar{T}}+V^\star_0-V^\star_{\bar{T}}, \pi_{\bar{T}}\right)\Bigg) ,
    \end{align*}
where the second line comes from the fact that the q-value for $\pi^\star$ is greater than the q-value for $\pi$.
    Now 
    \begin{align*}
        U(\pi^\star_{\bar{T}},b = q(b)\pi^\star_{\bar{T}} - p(b)= q(b)(\pi^\star_{\bar{T}}-b) +\int_0^b q(s) \dd s , 
    \end{align*}
    therefore 
    \begin{align*}
        U(\pi^\star_{\bar{T}},\pi^\star_{\bar{T}})  - U(\pi^\star_{\bar{T}},\pi_{\bar{T}})
         = \int_{\pi_{\bar{T}}}^{\pi^\star_{\bar{T}}} q(s) \dd s  - q(\pi_{\bar{T}})(\pi^\star_{\bar{T}}-\pi_{\bar{T}})=D_Q(\beta(\pi^\star),\beta(\pi)). 
    \end{align*}    
\end{proof}

\begin{tcolorbox}
    \begin{theorem}[Bretagnolle-Huber inequality, taken from~\cite{lattimore2020bandit}]
Let $P$ and $Q$ be probability measures on the same measurable space $(\Omega, \mathcal{F})$, and let $A \in \mathcal{F}$ be an arbitrary event. Then,
\begin{align*}
    P(A)+Q\left(A^c\right) \geq \frac{1}{2} \exp (-\mathrm{D}(P, Q)),
\end{align*}
where $A^c=\Omega \backslash A$ is the complement of $A$.
\end{theorem}
\end{tcolorbox}

Our initial intuition is to apply the Bretagnolle-Huber argument to two environments with nearly identical primitives, demonstrating that the resulting policies would exhibit similar behavior in both environments, and then derive a lower bound using Lemma~\ref{lemma:bregman-lower}.
However, we can simplify this approach by considering two degenerate cases where $k$
 is constant.
In doing so, the problem reduces to the static case, with the important distinction that feedback arrives in batches and is aggregated.
Since allowing the algorithm to disaggregate the data freely can only improve its performance, any lower bound derived under this relaxation remains valid for the original problem.
This reduction maps our setting to one of the scenarios analyzed in~\cite{weed2016online}, for which a lower bound has already been established. More precisely, we can apply Theorem 4 with $\alpha=1$ to get the next result. 
\begin{tcolorbox}
    \begin{theorem}
There exist a constant $C$ such that
    for any algorithm, and any $N$, $Regret(N)\geq C\log(N)$.
\end{theorem}

\end{tcolorbox}

\end{document}